%% file: AISTATS2026PaperPack/main.tex
\documentclass[twoside]{article}

\usepackage[accepted]{AISTATS2026PaperPack/aistats2026} 
\usepackage{thm-restate}
\usepackage{lipsum} 
\usepackage{amsmath}
\usepackage{colortbl}
\usepackage{wrapfig}
\usepackage{subcaption}
\usepackage{siunitx} 
\usepackage{subcaption}   
\usepackage{caption}
\usepackage{subcaption}
\usepackage{pgffor}
\usepackage{pgfplots}
\usepackage{wrapfig}
\pgfplotsset{compat=1.18}
\usepgfplotslibrary{groupplots} 
\usepackage{amssymb}

\input{AISTATS2026PaperPack/math_commands}
\usepackage{hyperref}
\usepackage{url}
\usepackage{wrapfig} 
\usepackage{graphicx}
\usepackage{subcaption} 

\usepackage{times} 
\usepackage{helvet} 
\usepackage{courier} 
\usepackage{graphicx} 
\usepackage[table]{xcolor}
\usepackage{natbib} 
\usepackage{caption} 
\usepackage{comment}
\usepackage{todonotes}
\usepackage{xcolor} 
\usepackage[most]{tcolorbox}
\usepackage{thmtools}
\usepackage{thm-restate}
\usepackage{newfloat}
\usepackage{listings}
\DeclareCaptionStyle{ruled}{labelfont=normalfont,labelsep=colon,strut=off} 
\usepackage{bibentry}
\usepackage{graphicx}

\usepackage{booktabs}
\usepackage{mhchem}
\usepackage[most]{tcolorbox}
\usepackage{amsmath, amssymb}

\usepackage[ruled,vlined,linesnumbered]{algorithm2e}

\usepackage{tcolorbox}
\tcbuselibrary{breakable}  
\usepackage{wrapfig}

\usepackage{xcolor}
\definecolor{Ebb}{rgb}{0.917,0.901,0.901}
\usepackage{amsmath}
\usepackage{amsthm}
\usepackage{amsfonts}
\usepackage{amssymb}
\usepackage{amsmath}
\usepackage{graphicx}
\usepackage{subcaption}
\usepackage{mathtools}
\usepackage{multirow}
\usepackage{adjustbox}
\usepackage{amsthm}
\usepackage{comment}
\usepackage{tcolorbox, xcolor}
\usepackage[toc, page]{appendix}
\usepackage{graphicx}
\usepackage{mathtools}
\input{macros}

\usepackage{amsmath}

\newcommand{\alertbyPC}[1]{{\color{magenta} PC:#1}}

\begin{document}

\runningauthor{Chanda et al.}

\twocolumn[
\renewcommand{\thefootnote}{\fnsymbol{footnote}}%
\setcounter{footnote}{0}%
\aistatstitle{\papertitle}
\aistatsauthor{Prateek Chanda\textsuperscript{1}\footnotemark[1]\footnotemark[2] \And Prayas Agrawal\textsuperscript{3}\footnotemark[1] \And Karthik S.\ Gurumoorthy\textsuperscript{4} \AND Ganesh Ramakrishnan\textsuperscript{1,2} \And Bamdev Mishra\textsuperscript{5} \And Pratik Jawanpuria\textsuperscript{2}}
{\centering
\def\And{\quad\quad\allowbreak}
\parbox{0.95\textwidth}{\centering
{\normalsize
\textsuperscript{1} Department of Computer Science and Engineering, Indian Institute of Technology Bombay\\
\textsuperscript{2} Centre for Machine Intelligence and Data Science, Indian Institute of Technology Bombay\par
}\And
\textsuperscript{3} Microsoft Research India \And
\textsuperscript{4} Walmart Global Tech, India\And
\textsuperscript{5} Microsoft India}\par}
\vskip 0.3in plus 2fil minus 0.1in
]
\renewcommand{\thefootnote}{\fnsymbol{footnote}}%
\setcounter{footnote}{0}%

\footnotetext[1]{Equal contribution.}
\footnotetext[2]{Correspondence to: \texttt{prateekch@cse.iitb.ac.in}.}

\renewcommand{\thefootnote}{\arabic{footnote}}%
\setcounter{footnote}{0}%
\begin{abstract}
  
Selecting prototypical examples from a source distribution to represent a target data distribution is a fundamental problem in machine learning. Existing subset selection methods often rely on implicit importance scores, which can be skewed towards majority classes and lead to low-quality prototypes for minority classes. We present $\methodprop$, a novel subset selection framework that minimizes the
optimal transport (OT) distance between a uniformly weighted prototypical distribution
and the target distribution.
While intuitive, this formulation leads to a cardinality-constrained maximization of a \emph{super-additive} objective, 
which is generally intractable to approximate efficiently. To address this, we propose a principled reformulation of the OT marginal constraints, yielding a partial optimal transport-based submodular objective. We prove that this reformulation enables a greedy algorithm with a $(1-1/e)$ approximation guarantee relative to the original super-additive maximization problem. Empirically, we showcase that enforcing uniform prototype weights in UniPROT consistently improves minority-class representation in imbalanced classification benchmarks without compromising majority-class accuracy. In both finetuning and pretraining regimes for large language models under domain imbalance, UniPROT enforces uniform source contributions, yielding robust performance gains. Our results establish UniPROT as a scalable, theoretically grounded solution for uniform-weighted prototype selection. Our
code is publicly available at GitHub\footnote{Code: \url{https://github.com/efficiency-learning/UniPROT}}

\end{abstract}

    \addtocontents{toc}{\protect
    \setcounter{tocdepth}{-1}}

\section{Introduction}
Prototype selection is a fundamental problem in representation learning. The goal is to identify a subset of representative elements from a source set or distribution that faithfully summarizes a given target set or distribution. In cases where the source and target sets coincide, the problem reduces to the well-known {medoids selection} task. Prototypical examples have proven useful in a variety of applications, including domain understanding via summarization \citep{schlegel17a,chen2019looks}, identifying anomalies \citep{Kawano22a}, positive-unlabeled learning \citep{Dhurandhar20a,riaz2023partial}, and efficient training of deep models \citep{MirzasoleimanCraig20a, killamsetty2021grad, Kothawade21,zheng2023coveragecentric,liu2024tsds,tan2025data}. 

Recent prototype selection algorithms usually select a prototypical set of size $k$ whose underlying distribution is \textit{closest} to the target distribution, as measured by a chosen divergence or distance metric between probability distributions. Prior works  \citep{kim16a,gurumoorthy19a,gurumoorthy2021spot} have explored metrics such as maximum mean discrepancy (MMD) and optimal transport (OT) distance (Wasserstein distance) to quantify their closeness. Interestingly, classical problems like submodular facility location \citep{lin2011class,krause2014submodular} or $k$-medoids problems may be viewed as special cases of OT based prototype selection. Submodular optimization has been widely adopted in this context due to its favorable approximation guarantees and scalability.
    
    


Prototype selection algorithms often yield a \textit{weighted} subset of representative instances, where the weights reflect the relative importance of each exemplar and are typically inferred implicitly during the selection process. However, for interpretability, \textit{uniform weighting} is generally preferred, as disproportionate emphasis can bias human perception \citep{solso2017cognitive}. This issue is particularly pronounced in long tailed class distributions, where minority classes may receive lower-weighted prototypes, leading to unfair or under-representative selections. 
Therefore, it is beneficial to design methods that promote equal importance among
selected prototypes. This ensures balanced and interpretable representation, particularly
in long-tailed settings.



In this work, we focus on the problem of selecting a uniformly weighted prototypical set which is closest to the target set under the optimal transport metric. We begin by showing that popular formulations such as submodular facility location and $k$-medoids inherently produce weighted prototype sets when viewed through the lens of OT. Motivated by this observation, we propose a novel subset selection problem aimed at identifying an equally weighted prototypical set. Although intuitively appealing, the proposed formulation corresponds to a monotone, non-negative, super-additive maximization problem, which is not directly amenable to greedy optimization. To address this challenge, we design a tight, monotone, non-negative, \emph{submodular surrogate objective} that approximates the original super-additive problem. This reformulation enables the use of greedy algorithms with strong approximation guarantees. Furthermore, we prove that the same theoretical guarantees hold for the original super-additive maximization problem, thereby validating the effectiveness of our approach.
    
Our main contributions are summarized as follows:
\begin{itemize}
\item We formalize uniform prototype selection as a super-additive maximization problem under the cardinality constraint of selecting $k$ prototypical examples from a source set and introduce a submodular reformulation with provable guarantees.
    
\item We show that the proposed problem corresponds to a  monotone, non-negative, super-additive maximization problem under cardinality constraint. To the best of our knowledge, efficient algorithms with provable guarantees are not known for this  class.

\item We prove that our submodular reformulation is equivalent to the original super-additive problem with cardinality constraint $k$, using which we establish a $(1-1/e)$ approximation guarantee for the latter. We also develop an efficient greedy algorithm whose computational cost is comparable to that of solving the submodular $k$-medoids problem. 

\item We demonstrate the utility of selecting uniformly weighted prototypical set in applications such as long tailed image classification and high quality mini-batch selection for large language model (LLM) training. Our proposed method, \textbf{$\methodprop$}, outperforms existing prototype selection methods across various  benchmark datasets both in terms of solution quality and computational efficiency.
\end{itemize}

\section{Preliminaries}


Let $\S \coloneq \{\bx_{i}\}_{i=1}^{m}$ and $\T \coloneq \{\by_{j}\}_{j=1}^{n}$ be the source and the target datasets, respectively, where $\bx_i\in\mathcal{X}$ and $\by_j\in\mathcal{Y}$. 
The corresponding source and target empirical distributions may be written as $\mu= \sum_{i=1}^{m}\bmu_{i}\delta_{\bx_i}$ and $\nu= \sum_{j=1}^{n}\bnu_{j}\delta_{\by_j}$ where $\bmu_i$ and $\bnu_j$ denote the mass associated with $\bx_i$ and $\by_j$, respectively, and $\delta_{\bz}$ denote the Dirac measure centered at $\bz$. 
If $\mu$ and $\nu$ are probability distributions, $\bmu\in\Delta_m$ and $\bnu\in\Delta_n$ where $\Delta_{m}\coloneq \{\bz\in \RR^{m}_{+}\mid \bz^\top \bone = 1\}$ and $\bone$ denote the vector of ones of appropriate size. For $\bz\in\RR_{+}^m$, let ${\rm supp}(\bz)=\{i\in [m]\mid \bz_i > 0\}$. For any subset $\P \subseteq \S$, let $\I_{\P}$ be the set of indices corresponding to the points $\bx \in \P$, i.e. $\I_{\P} = \{i: \bx_i \in \P\}$. Let $\bZ\left(\I_{\P},:\right)$ denote the sub-matrix of $\bZ$ containing rows corresponding to the indices in $\I_{\P}$, and when $\I_{\P} = \{i\}$ is a singleton set, we represent the $i$-th row the matrix $\bZ$ as $\bZ(i,:)$. Lastly, let $[m]=\{1,2,\dots, m\}$ for $m\in \mathbb{N}$. 


\textbf{Optimal Transport} (OT) problem \citep{KatoroOT,peyre2019computational} seeks a transport plan $\mgamma$ that minimizes the total cost of moving mass from a source distribution $\mu$ to a target distribution $\nu$: 
\begin{equation}\label{eqn:OT-min}
     {\rm OT}_{\min}(\bmu,\bnu)  = \minop_{\mgamma \in \Gamma(\bmu,\bnu) }\inner{\bC, \mgamma}, 
\end{equation}
where $\bmu\in\Delta_m, \bnu\in\Delta_n,\Gamma(\bmu,\bnu)=\{\mgamma \in \RR_{+}^{m \times n}\mid \mgamma\bone
= \bmu, \mgamma^{\top}\bone = \bnu\}$ is the set of admissible couplings and $\bC\in\RR^{m\times n}$ denote a cost matrix induced by a ground cost function $c:\mathcal{X}\times \mathcal{Y}\rightarrow\RR_{+}: (\bx,\by)\mapsto c(\bx,\by)$ such that $\bC_{ij}=c(\bx_i,\by_j)$. Hence, $\bC_{ij}$ is the cost of transport a unit mass from $\bx_{i}$ to $\by_{j}$. 
When $c$ is a distance (e.g., $\ell_1$ or $\ell_2$ distance), OT cost induces the Wasserstein distance between the probability distributions $\bmu$ and $\bnu$. In doing so, OT lifts the geometry from the underlying sample space to the space of probability measures, enabling a rich geometric framework for comparing distributions. 

Let $\bS\in\RR_{+}^{m\times n}$ be a similarity matrix defined via $\bS_{ij} = \beta - \bC_{ij}$, where the constant $\beta > \max_{ij} \bC_{ij}$ ensures non-negativity of $\bS$. Then, the following maximization problem is equivalent to (\ref{eqn:OT-min}) as they have the same optimal solution(s):
\begin{equation}\label{eqn:OT}
     {\rm OT}(\bmu,\bnu)  = \maxop_{\mgamma \in \Gamma(\bmu,\bnu)}\inner{\bS, \mgamma} = \beta - {\rm OT}_{\min}(\bmu,\bnu). 
\end{equation}


\textbf{Partial Optimal Transport} (POT) generalizes classical OT by allowing only a subset of the source and/or target mass to be matched \citep{benamou2015iterative,chapel20a,nguyen24a}. A commonly studied variant is the semi-relaxed formulation, suited for unbalanced settings where the target distribution may contain excess mass. Using similarity matrix $\bS$, it can be expressed as:
\begin{equation}\label{eqn:pot}
        {\rm POT}(\bmu,\bnu) = \maxop_{\substack{\mgamma \in \Gamma_{\leq}(\bmu,\bnu)} }\inner{\bS, \mgamma} \left(= \beta\bmu^\top\bone - \minop_{\substack{\mgamma \in \Gamma_{\leq}(\bmu,\bnu)} }\inner{\bC, \mgamma}\right),
\end{equation}
where $\Gamma_{\leq}(\bmu, \bnu)=\{\mgamma \in \RR^{m \times n} \mid \mgamma \geq 0, \mgamma\bone = \bmu, \mgamma^{\top}\bone \leq \bnu\}$. 
It should be noted that ${\rm POT}(\bmu,\bnu) = {\rm OT}(\bmu,\bnu)$ when $\bmu^\top \bone = \bnu^\top \bone$, i.e., the source and the target distributions have equal mass. 

\textbf{Submodularity} is a characteristic of set functions that capture diminishing returns: for any sets $A \subseteq B \subseteq V$ and element $u \notin B$, a set function $F$ is submodular if \(F(A \cup \{u\}) - F(A) \geq F(B \cup \{u\}) - F(B).\) 
The term $F(u \mid A) := F(A \cup \{u\}) - F(A)$ denotes the \textit{marginal gain} of adding $u$ to $A$. A function is \textit{monotone} if $F(A) \leq F(B)$ whenever $A \subseteq B$. 
We provide an alternative definition of submodularity in Section \ref{defn:submodularity ratio}. For maximizing a non-negative monotone submodular function under a cardinality constraint, i.e., $\max_{S \subseteq V, |S| \leq k} F(S)$, the greedy algorithm achieves a $(1 - 1/e)$ approximation to the optimal value~\citep{nemhauser1978analysis}.

In the next section~\ref{sec:approach}, we leverage POT to construct a tractable 
submodular surrogate of the uniform prototype selection problem.

\section{Proposed Approach} \label{sec:approach}
\textbf{Problem setup:} Given a source set $\S$ and a target set $\T$, let $\P$ be a candidate prototypical set $\P \subseteq \S$ such that $|\P| \leq k$. 
The empirical distribution corresponding to set $\P$ may be expressed as $\mu_{\P}=\sum_{i=1}^m (\bmu_{\P})_{i} \delta_{\bx_i}$, where $\bmu_{\P}\in\Delta_m$ and $(\bmu_{\P})_{i} = 0\ \forall \bx_i\notin\P$.  
Our aim is to find the best prototypical set $\P^{*}$ such that 
\begin{itemize}
    \item all element of $\P^{*}$ have equal mass (importance) in the corresponding  distribution $\mu_{\P^{*}}=\sum_{i=1}^m (\bmu_{\P^{*}})_{i} \delta_{\bx_i}$, i.e., $(\bmu_{\P^{*}})_{i} = 1/|\P^{*}|\ \forall \bx_i\in\P^{*}$, and 
    \item $\bmu_{\P^{*}}$ is \textit{closest} to the underlying target distribution $\nu$ under the optimal transport (OT) distance metric. 
\end{itemize}

We note that popular submodular subset selection problems such as facility location or exemplar based clustering ($k$-medoids) may be viewed as selecting prototypes using the OT metric. In our setup, their optimization objective may be written as
\begin{equation}\label{eqn:spot}
    \begin{split}
        &\max_{\P\subseteq \S, |\P|\leq k} l(\P), 
        \\ 
        \textup{where \ } l(\P) &\coloneqq \max_{\mgamma \in \RR^{m\times n}_{+}, \mgamma^\top\bone = \bnu, {\rm supp}(\mgamma\bone)\subseteq \P }\inner{\bS, \mgamma}
    \end{split}
\end{equation}
and $\bnu\in\Delta_n$ is the given target set distribution, usually set as $\bnu=\bone/n$. 
As the objective $l(\P)$ may also be written as $l(\P) = \max_{\bmu\in\Delta_{m}, {\rm supp}(\bmu) \subseteq \P} {\rm OT}(\bmu,\bnu)$, solving (\ref{eqn:spot}) implicitly involves learning the underlying distribution of $\P$. 
\begin{figure}
\centering
\includegraphics[width=0.85\linewidth]{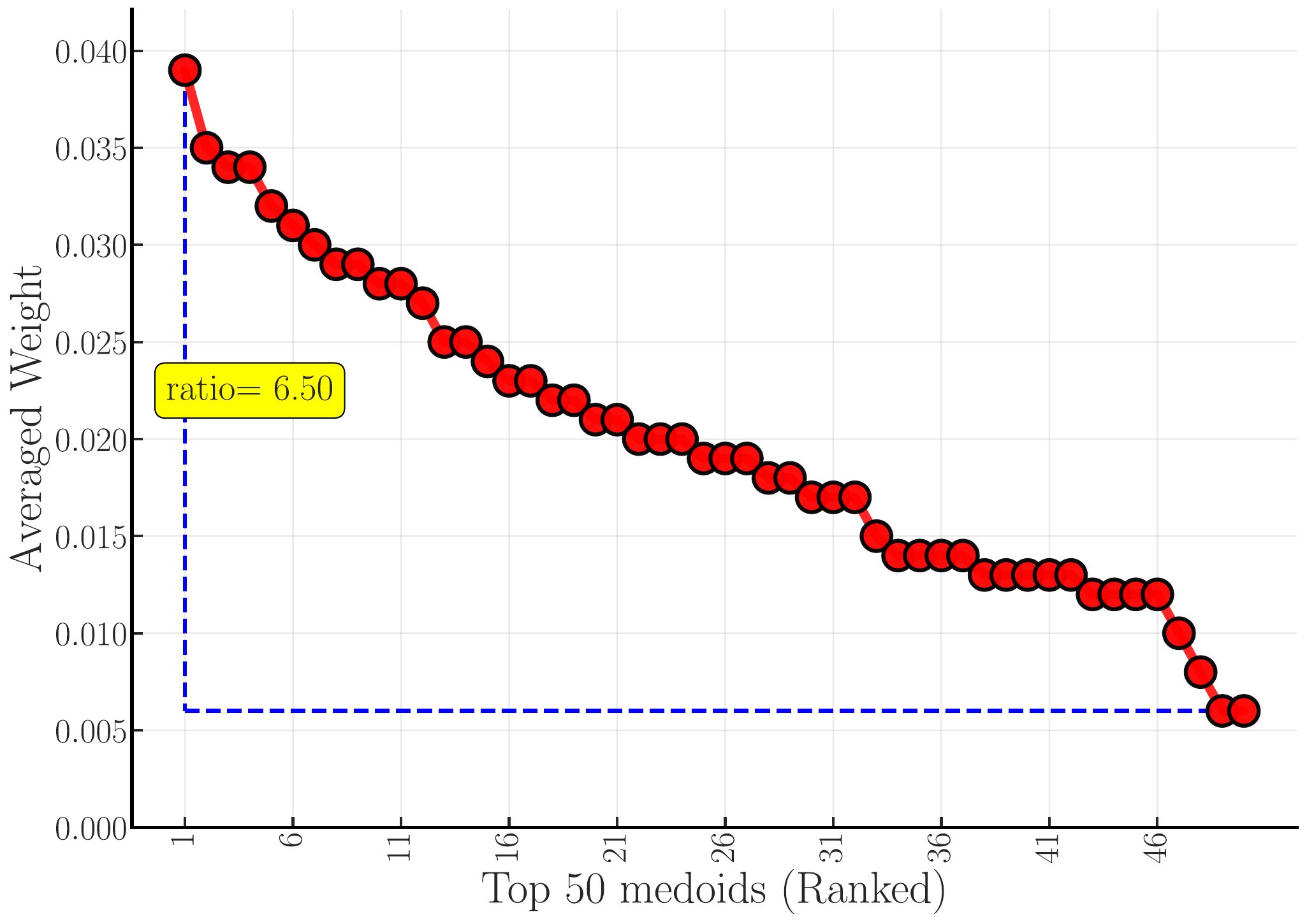}
\caption{\small $k$-medoids (\ref{eqn:spot}) consistently learn skewed weights for prototypes on \textsc{CIFAR10}. The plot shows ranked weights of prototypes, averaged over 5 runs. 
}
\label{Fig:introplug}
\end{figure}

In particular, if $\hat{\P}$ is a solution of (\ref{eqn:spot}) with the corresponding $\mgamma_{\hat{\P}}$ such that $l(\hat{\P}) = \inner{\bS,\mgamma_{\hat{\P}}}$, then $\bmu_{\hat{\P}}= \mgamma_{\hat{\P}}\bone$. If the learned $\bmu_{\hat{\P}}$ is skewed, it implies that some prototypes have higher mass (importance) than the others. Empirically, this is commonly observed as shown in Figure \ref{Fig:introplug}. When one aims to understand the target set $\T$ via the prototypical set $\hat{\P}$ obtained via (\ref{eqn:spot}),  $\hat{\P}$ and $\bmu_{\hat{\P}}$ together provide a representation of $\T$. 
However, weighted prototypes are hard to interpret, especially for human-in-the-loop scenarios. 
Skewed prototypical distributions also imply that prototypes receiving low weights contribute less towards the overall objective (\ref{eqn:spot}) and may be less influential exemplars. The minority classes often suffer in such cases as they typically receive low weights. 
Hence, uniformly weighted prototype selection algorithms are desired for fair, unbiased representation and better understanding of datasets.

\subsection{Uniform Prototype Selection via Optimal Transport}
We propose to alleviate the issue of learning unequally weighted exemplars by enforcing the prototypical distribution to be uniform in the objective. Thus, the proposed uniformly weighted prototype selection problem is as follows:
\begin{equation}\label{eqn:proposed_1}
\begin{aligned}
\max_{\P\subseteq \S,\ |\P|\leq k}\; g(\P)
\qquad\text{s.t.}\qquad
 g(\P)=\maxop_{\mgamma \in \Gamma(\bone_{\P}/|\P|,\bone/n)}\inner{\bS, \mgamma}.
\end{aligned}
\end{equation}
Here, $g(\P)$ denotes the OT objective with uniform marginals $\bmu_{\P}=\bone_{\P}/|\P|$ and $\bnu=\bone/n$. Here, $\bone_\P\in\{0,1\}^m$ represents the set $\P$, i.e., $(\bone_\P)_i = 1$ if $\bx_i \in \P$, else $0$. It should be noted that in the prototypical distribution $\bmu_{\P}$ corresponding to $\P$, all the exemplars $\bx\in\P$ are given equal mass in (\ref{eqn:proposed_1}). This implies that all the selected exemplars are equally important. We also note that the total mass assigned to the set $\S\setminus\P$ is $1 - \bmu_\P^\top \bone = 0$. 
Empirically, the target set distribution is usually considered uniform $\bnu= \bone/n$, but our analysis directly extends to non-uniform target set distributions as well. 
In order to analyze the properties of the objective in~(\ref{eqn:proposed_1}), we consider the following proxy problem:
\begin{equation}\label{eqn:proposed_2}
\begin{aligned}
        \max_{\P\subseteq \S, |\P|\leq k} h(\P)\coloneqq |\P|g(\P), \\ \textup{where}
        \hspace{4pt} h(\P) = {\rm OT}(\bmu_\P=\bone_\P,\bnu= |\P|\bone/n).
        \end{aligned}
\end{equation}
We observe that for a given $\P$, if $\mgamma^{*}_{g}$ is an optimal solution for computing $g(\P)$ in (\ref{eqn:proposed_1}), then $\mgamma_{h}^{*}= |\P|\mgamma^{*}_{g}$ is an optimal solution for
computing $h(\P)$ in (\ref{eqn:proposed_2}) (and vice-versa). Hence, we focus on Problem (\ref{eqn:proposed_2}) in the next lemma.
\begin{restatable}{lemma}{LemmaAdditivity}\label{lemma:properties} 
    The set function $h(\P): 2^{|\S|}\rightarrow \RR_{+}$, 
    defined in (\ref{eqn:proposed_2}), satisfies the following properties:
    \begin{enumerate}
        \item Non-negativity: $h(\P)\geq 0\ \forall \P\subseteq \S$.

        \item Monotonicity:
            $h(\P_{2}) \geq h(\P_{1})\ \forall \P_{1}\subseteq \P_{2}\subseteq
            \S$.

        \item Super-additivity over disjoint sets:
            $h(\P_{1}\cup \P_{2}) \geq h(\P_{1}) + h(\P_{2})\ \forall \P_{1}\cap
            \P_{2}= \phi$.
    \end{enumerate}
\end{restatable}
The proof of the above result is provided in Appendix~\ref{ref:TheoreticalResults}. 
\begin{rem}
Super-additivity is conceptually aligned with increasing returns (supermodularity), just as sub-additivity relates to diminishing returns (submodularity). 
Hence, maximizing a monotone, non-negative, super-additive function via the greedy algorithm and obtaining approximation guarantees is challenging. To address this, we propose a tight, non-negative, monotone submodular reformulation of Problem~(\ref{eqn:proposed_2}) in the next section.
\end{rem}

\subsection{Submodular Reformulation of (\ref{eqn:proposed_2})}
We propose the following partial optimal transport (POT) based reformulation of Problem~(\ref{eqn:proposed_2}):
\begin{equation}\label{eqn:proposed_3}
\begin{gathered}
\max_{\P\subseteq \S,\, |\P|\leq k}\; f(\P),\\
\begin{aligned}
\textup{where}\quad f(\P) &\coloneqq {\rm POT}(\bmu_{\P}=\bone_\P,\bnu= k\bone/n) \\
&= \maxop_{\mgamma \in \Gamma_{\leq}(\bone_\P,k\bone/n)}\inner{\bS,\mgamma}.
\end{aligned}
\end{gathered}
\end{equation}
We note that computing $f(\P)$ in (\ref{eqn:proposed_3}) is a semi-relaxed optimal transport problem in which the source marginal is tight $(\bmu_{\P}=\bone_\P)$ but the target side marginal constraint is relaxed $(\bnu\leq k\bone /n)$. In contrast, computing $h(\P)$ in (\ref{eqn:proposed_2}) is an OT problem. By relaxing the target side constraint in (\ref{eqn:proposed_3}), it is easy to see that $f(\P) \geq h(\P), \forall \P$ with $|\P| \leq k$. We also note that for the sets $\P$ with cardinality $k$, $f(\P)={\rm OT}(\bmu_{\P}=\bone_\P,\bnu= |\P|\bone/n) = h(\P)$. Our proposed reformulation allows Problem~(\ref{eqn:proposed_3}) to have certain desirable properties as summarized in our next result. 
\begin{restatable}{lemma}{lemmaSubmod}\label{lemma:submod_3}
The optimization problem defined in (\ref{eqn:proposed_3}) is a non-negative,  monotone, submodular maximization problem subject to cardinality constraint $k$.
\end{restatable}
Lemma~\ref{lemma:submod_3} implies that the classical greedy solution provides the $(1-1/e)$ approximation guarantee for (\ref{eqn:proposed_3}).

The following lemma proves that Problem~(\ref{eqn:proposed_3}) is a tight reformulation of Problem~(\ref{eqn:proposed_2}) and hence we may equivalently solve the relaxed  (\ref{eqn:proposed_3}) instead of (\ref{eqn:proposed_2}), for selecting uniformly weighted prototypes. 

\begin{restatable}{lemma}{equivalencelemma} \label{lemma:equivalence} 
Let $\P^{*}$ of cardinality $k$ be an optimal solution of (\ref{eqn:proposed_2}). Then $\P^{*}$ is also an optimal solution of (\ref{eqn:proposed_3}), and vice-versa.
\end{restatable}

 The above analysis ensures that the same approximation guarantee holds for  for the super-additive maximization problem (\ref{eqn:proposed_2}) as stated below. 
\begin{restatable}{lemma}{submodulargain}\label{lemma:submodulargain}
    Let $\hat{\P}$ be the classical greedy solution of (\ref{eqn:proposed_3}) with $|\hat{\P}|=k$. Let $\mathrm{OPT}=h(\P^{*})$, where $\P^{*}$ is an optimal solution of (\ref{eqn:proposed_2}). Then, $h(\hat{\P}
    )\geq (1-1/e)\mathrm{OPT}$.
\end{restatable}

The proof of the above results for Lemmas \ref{lemma:submod_3},\ref{lemma:equivalence},\ref{lemma:submodulargain} are provided in Appendix~\ref{ref:TheoreticalResults}. 

\subsection{Computationally efficient approximate greedy algorithm for (\ref{eqn:proposed_3})}
The classical greedy algorithm \citep{nemhauser1978analysis} for (\ref{eqn:proposed_3}) begins with the empty set $\P_0=\emptyset$. At iteration $i+1$, it selects an element $\bx^{*}$ with highest marginal gain, i.e., $\bx^{*} = \argmax\nolimits_{\bx\in\S\setminus\P_{i}} f(\bx|\P_{i})$, and updates $\P_{i+1} = \P_{i}\cup\{\bx^{*}\}$. 
For computing the marginal gains of all elements in $\S\setminus\P_{i}$, $(m-i)$ POT problems (\ref{eqn:pot}) need to be solved in the $(i+1)$-th classical greedy iteration. While this number may be reduced using lazy (stochastic) greedy \citep{Minoux78a,mirzasoleiman2015lazier}, the per-iteration cost remains high {for large $k$}. 
Hence, we propose a computationally efficient approximate marginal gain estimator for (\ref{eqn:proposed_3}). 

In this regard, for a given $\P\subset \S$ such that $|\P|<k$, let $\bx_j\in\S\setminus\P$. For notational convenience, let $\P' = \P\cup\{\bx_j\}$ and let $\mgamma_{\P} = \argmax\nolimits_{\mgamma\in\Gamma_{\leq}(\bone_\P,k\bone/n)}\inner{\bS,\mgamma}$.  
We denote by $\hat{\mgamma}_{\P'}$ a feasible POT coupling between the sets $\P'$ and $\T$ such that 
$\hat{\mgamma}_{\P'}\left(\I_\P,:\right) = \mgamma_\P\left(\I_\P,:\right)$ and $\hat{\mgamma}_{\P'}(j,:)=\bv^\top$, where $\bv\in\RR^n_{+}$ is a variable. 
We next construct an estimator of $f(\P')$ as $\hat{f}(\P') = \max_{\hat{\mgamma}_{\P'}\in\Gamma_{\leq}(\bone_{\P'},k\bone/n)}\inner{\bS,\hat{\mgamma}_{\P'}}$, which is essentially a constrained optimization over $\bv$. 
Our approximate marginal gain function for (\ref{eqn:proposed_3}) is as follows: 
\begin{equation}\label{eqn:approx-incremental-gain}
\begin{aligned}
    \hat{f}(\bx_j\mid\P)
    &= \hat{f}(\P') - f(\P) \\
    &= \max_{\substack{\bv\in\RR_{+}^{n}
                        \bv^\top\bone = 1
                        \bv\leq \tfrac{k}{n}\bone - \mgamma_{\P}^\top\bone}}
    \inner{\bS(j,:),\bv^\top}
\end{aligned}
\end{equation}
For a given $\mgamma_{\P}$, Problem~(\ref{eqn:approx-incremental-gain}) has a closed form expression which involves sorting the vector $\bS(j,:)$, i.e., $O(n\log n)$ computation. 
The following result quantifies the approximation guarantee corresponding to the greedy solution obtained using the proposed approximate marginal gain (\ref{eqn:approx-incremental-gain}). 

\begin{restatable}{lemma}{approximateMarginalGain}\label{lemma:approximateMarginal}
Let $\alpha_{j,\min}$ denote $\frac{1}{\lfloor n/k \rfloor}$ times the sum of the $\lfloor n/k \rfloor$ smallest entries of the vector $\bS(j,:)$, and let $\alpha_{j,\max}$ denote $\frac{1}{\lfloor n/k \rfloor}$ times the sum of the $\lfloor n/k \rfloor$ largest entries of $\bS(j,:)$. Define $\alpha = \min_{j\in [m]} \frac{\alpha_{j,\min}}{\alpha_{j,\max}}$. Let $\hat{\P}$ be the solution returned by the greedy algorithm for (\ref{eqn:proposed_3}), where the proposed approximate marginal gain function (\ref{eqn:approx-incremental-gain}) is used in each iteration and $|\hat{\P}|=k$. Then, $f(\hat{\mathcal{P}}) = h(\hat{\mathcal{P}}) \geq (1 - e^{-\alpha})\, \mathrm{OPT}$, where $\mathrm{OPT} = f(\mathcal{P}^*) = h(\mathcal{P}^*)$ and $\mathcal{P}^*$ is an optimal solution to (\ref{eqn:proposed_3}) with $|\mathcal{P}^*|=k$.

\end{restatable}

\begin{figure*}[t]
\centering
\begin{tcolorbox}[
  colback=gray!20, colframe=gray!20,
  left=2mm, right=2mm, top=1mm, bottom=1mm,
  boxsep=0pt, width=1\linewidth
]
  \begin{minipage}[t]{0.51\linewidth}
    \begingroup
    \setlength{\parindent}{0pt}\setlength{\parskip}{0pt}
    \noindent\hspace*{-0.6em}\rule{\dimexpr\linewidth+1em\relax}{1.2pt}\par
    \noindent\hspace*{-1em}{%
      \captionsetup{type=algocf,
        width=\dimexpr\linewidth+1em\relax,
        justification=raggedright,singlelinecheck=false,
        aboveskip=0pt,belowskip=0pt,skip=0pt}
      \captionof{algocf}{\protect\methodprop}%
      \label{alg:greedy}%
    }\par
    \noindent\hspace*{-0.6em}\rule{\dimexpr\linewidth+1em\relax}{1.2pt}\par
    \endgroup

    \textbf{Input:} Similarity matrix $\bS$ between $\S$ and $\T$, number of prototypes required $k$, entropic regularization parameter $\lambda$\\
    \textbf{Output:} Uniformly weighted prototypical set $\mathcal{P}_k\subseteq\S$ of $\T$\\[0.25em]

    \noindent\textbf{1.} $\mathcal{P}_0 \gets \emptyset$\\
    \noindent\textbf{2.} \textbf{for} $i = 1$ \textbf{to} $k$ \textbf{do}\\
    \textbf{3.}\hspace*{1em}$\mgamma^{\ast}_{\mathcal{P}_i} \gets \argmax\nolimits_{\mgamma\in\Gamma_{\leq}(\bone_{\mathcal{P}_i}, k\bone_n/n)}\; \inner{\bS,\mgamma} - \lambda \inner{\mgamma,\ln{\mgamma}}$\\
    \textbf{4.}\hspace*{1em}$\bx^{*} \gets \argmax_{\bx\in\S\setminus\mathcal{P}_i} \hat{f}(\bx\mid\mathcal{P}_i)$ \hfill (Eq.~(\ref{eqn:approx-incremental-gain}))\\
    \textbf{5.}\hspace*{1em}$\mathcal{P}_{i+1}\gets \mathcal{P}_i\cup \{\bx^*\}$\\
    \noindent\textbf{6.} \textbf{end for}\\
    \noindent\textbf{7.} \textbf{return} $\mathcal{P}_k$
  \end{minipage}%
  \hfill
  \begin{minipage}[t]{0.45\linewidth}
    \vspace{1.3em} 
    \begin{itemize}
        \item \textbf{Steps 2–5} are iteratively executed until the cardinality constraint $k$ is satisfied.
        \item  \textbf{Step 3} At each iteration, it updates the transport plan $\mgamma^{\ast}_{\P_i}$ 
    and selects the element $\bx^{*}$ that maximizes the approximate marginal gain 
    defined in Eq.~(\ref{eqn:approx-incremental-gain}).
    \item \textbf{Step 4} selects the element $\bx^{*}$ that maximizes the approximate marginal gain $\hat{f}(\cdot)$ across the candidate search space $\S\setminus\P_i$. (\textit{Note}: For larger search space i.e. where $|\S|$ is large enough, we consider a \emph{Stochastic-Greedy} version instead of a Naive Greedy approach where the candidate search space is considered as $\R \subseteq \S \setminus \P_i$ with $nk^{-1}\log(1/\epsilon)$ elements which are selected uniformly at random.
    \end{itemize}
\end{minipage}
\end{tcolorbox}
\end{figure*}

The proof of the above result is provided in Appendix~\ref{ref:TheoreticalResults}. The key idea in the proof methodology is to lower bound the proposed approximate marginal gain (\ref{eqn:approx-incremental-gain}) as a fraction of the true marginal gain $f(\bx_j\mid\P)$. 

We observe that the proposed approximate marginal gain-based greedy algorithm yields theoretical guarantees for (\ref{eqn:proposed_2}) that are equivalent to those established for the maximization of an \(\alpha\)-weakly submodular function \citep{Das19a,elenberg2018restricted}.

\section{Algorithm Details}\label{ref:AlgorithmDetails}

Here, we present our algorithm for $\methodprop$ (Alg.~\ref{alg:greedy}). At each iteration, we first compute a partial optimal transport plan for the current set and then use it to evaluate the approximate marginal gain $\hat{f}(\cdot)$ over the remaining candidates, selecting the best next prototype. To solve the partial optimal transport subproblem in Step~3, we use Bregman iterations \citep{benamou2015iterative}.

\textbf{Computation Cost.} 
Overall, finding the (approximate) next best element $\bx^{*}$ requires solving a \textit{single} POT problem of dimension $(i+1)\times n$ along with $O((m-i)n\log n)$ additional computations. The POT problem can be efficiently solved in $O(i\cdot n)$ using the Bregman-Dykstra iterations \citep{benamou2015iterative} or the Sinkhorn algorithm \citep{cuturi2013lightspeed,chapel20a} by adding a small entropic regularization in (\ref{eqn:pot}). 
Hence, the computational cost of the proposed $\methodprop$ algorithm for selecting $k$ uniformly weighted prototypes is $O(kmn\log n)$. This cost can be reduced to $O(kmn)$ by utilizing an additional $m \times n$ memory to store the sorted rows of $\bS$, which is a one-time preprocessing step of $O(mn\log n)$ computations. 
Consequently, our algorithm selects equally important prototypes with an overall computational cost that closely matches that of the classical greedy algorithm for solving~(\ref{eqn:spot}). We term our approach $\methodprop$.

\begin{figure}[h!] 
  \centering
    \includegraphics[scale=0.12]{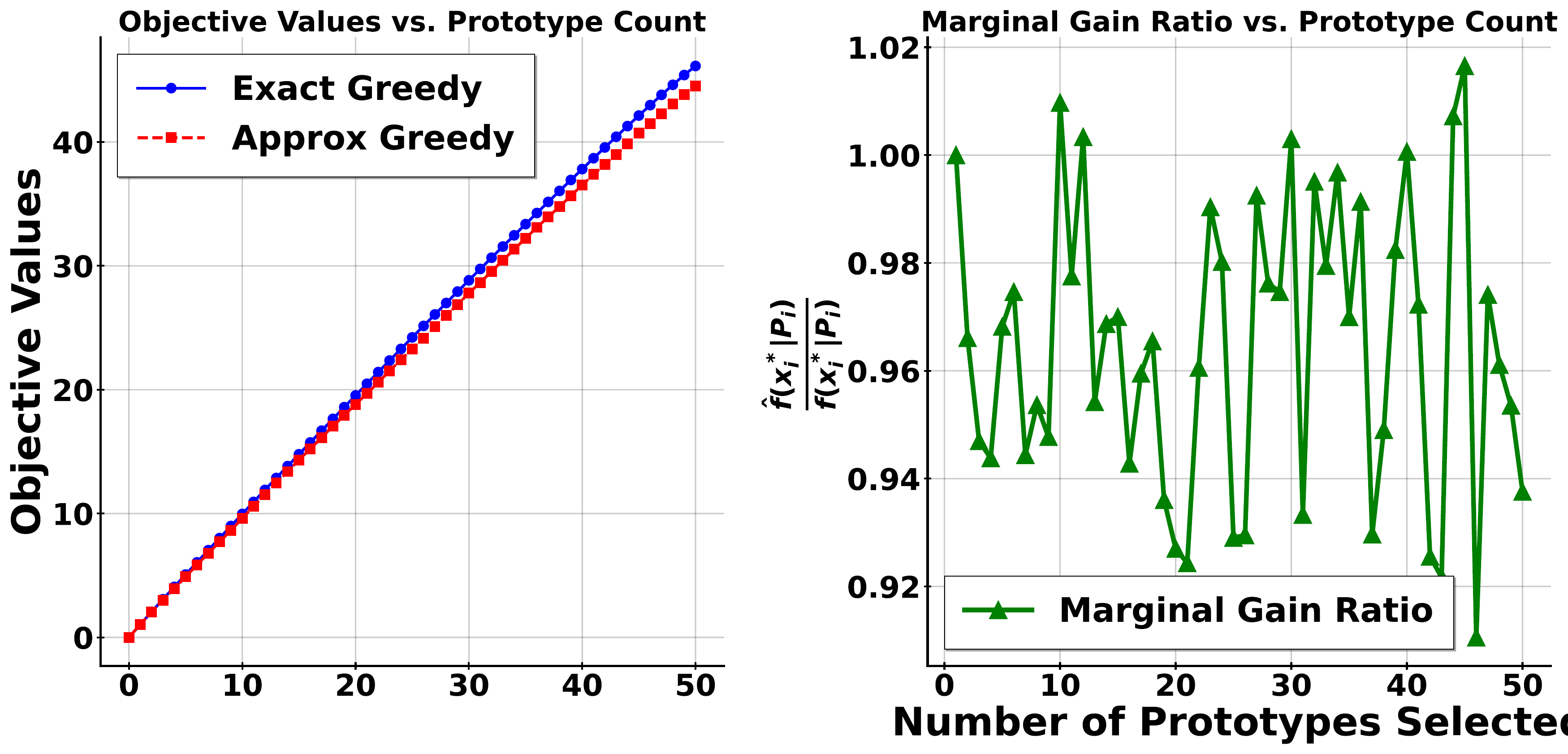}
  \caption{Exact marginal gain versus the proposed approximate marginal gain (\ref{eqn:approx-incremental-gain}) on MNIST.}
  \label{fig:approxgainbounds}
\end{figure}

\textbf{Comparing Approximate vs Exact Marginal gain:} We evaluate the effectiveness of the proposed computationally efficient approximate marginal gain (\ref{eqn:approx-incremental-gain}). For this, we compare the objective value $f(\P)$ under the two next element selection settings: (a) exact marginal gain, $\bx^{*} = \argmax\nolimits_{\bx\in\S\setminus\P} f(\bx|\P)$, and (b)  approximate marginal gain (\ref{eqn:approx-incremental-gain}), $\bx^{*} = \argmax\nolimits_{\bx\in\S\setminus\P} \hat{f}(\bx|\P)$. 

In Figure~\ref{fig:approxgainbounds}, we plot the objective values and the ratio of approximate marginal gain and exact marginal gain across greedy iterations. We observe that, across iterations, the two objectives are very close and the ratio  $\hat{f}(\bx_i^{*}|\P_i) / f(\bx^{*}_i|\P_i)$ is close to its maximum value $1$. This implies that the proposed approximate marginal gain (\ref{eqn:approx-incremental-gain}) of (\ref{eqn:proposed_3}) is a good computationally efficient alternative to exact marginal gain. We provide additional results on how different choices of entropic regularization as part of our implementation affects this approximation gap in Appendix \ref{supp:ablation_entropy}.

\section{Experimental Evaluation}

We now empirically validate the utility of $\methodprop$ based uniform weighted prototype selection in various domains. 

\subsection{Long Tailed Image Classification}
We assess the effectiveness of the representative samples selected by the weighted prototype selection method (\ref{eqn:spot}) and our proposed uniformly weighted variant, $\methodprop$, by evaluating the performance of the corresponding nearest prototype classifiers \citep{bien2011prototype,kim16a,gurumoorthy2021spot} in imbalanced multiclass classification setting. 

\textbf{Experimental Setup.} Let \( \S \) and \( \T \) denote source and target datasets, respectively, such that $\S\cap\T=\emptyset$. Source set $\S$ has same number of samples from all classes while the target set $\T$ exhibit a skewed class distribution. A skewed target set distribution simulate real-world scenarios involving non-trivial marginal shifts. 
The label information is not available during the prototype selection phase, {\em i.e.}, prototype selection is completely unsupervised. Let \( \P \subseteq \S \) be a candidate prototypical set intended to model the target dataset \( T \). 
After the set $\P$ is obtained, class labels of prototypical examples are now made available. 
We next parameterize a 1-nearest neighbor (1-NN) classifier with the prototypes in $\P$ (see Appendix~\ref{app:image-exp_details}) and using it to classify the data points in $\T$. Overall, the performance of the 1-NN classifier parameterized with $\P$ is an indicator of how representative $\P$ is of the target $\T$ \citep{bien2011prototype,gurumoorthy2021spot}. 

\textbf{Datasets.} We consider the $\texttt{MNIST}$ dataset and long-tailed versions of $\texttt{CIFAR-LT}$ \citep{krizhevsky2009learning}  datasets. The latter was obtained using the long-tail experimental setup of \citep{menon2021statistical}. For $\texttt{MNIST}$, we obtain a skewed target set distribution by ensuring that two (randomly) chosen class constitute $k\%$ (each) of $|\T|$ and the remaining $(100-2k)\%$ is spread uniformly over the other classes. Additional datasets are provided in Appendix \ref{supp:additional}.



\textbf{Results.} In Figure~\ref{fig:longtailclassifiation}, we observe that the proposed $\methodprop$ improves the minority class performance over $k$-medoids (\ref{eqn:spot}) which selects weighted prototypes \citep{gurumoorthy2021spot}. It should be noted that the learned weights of the latter were not employed during the inference stage as it deteriorates the performance. 

\begin{figure*}[ht]
    \centering
    \begin{subfigure}{0.48\textwidth}
        \centering
        \includegraphics[width=0.48\linewidth]{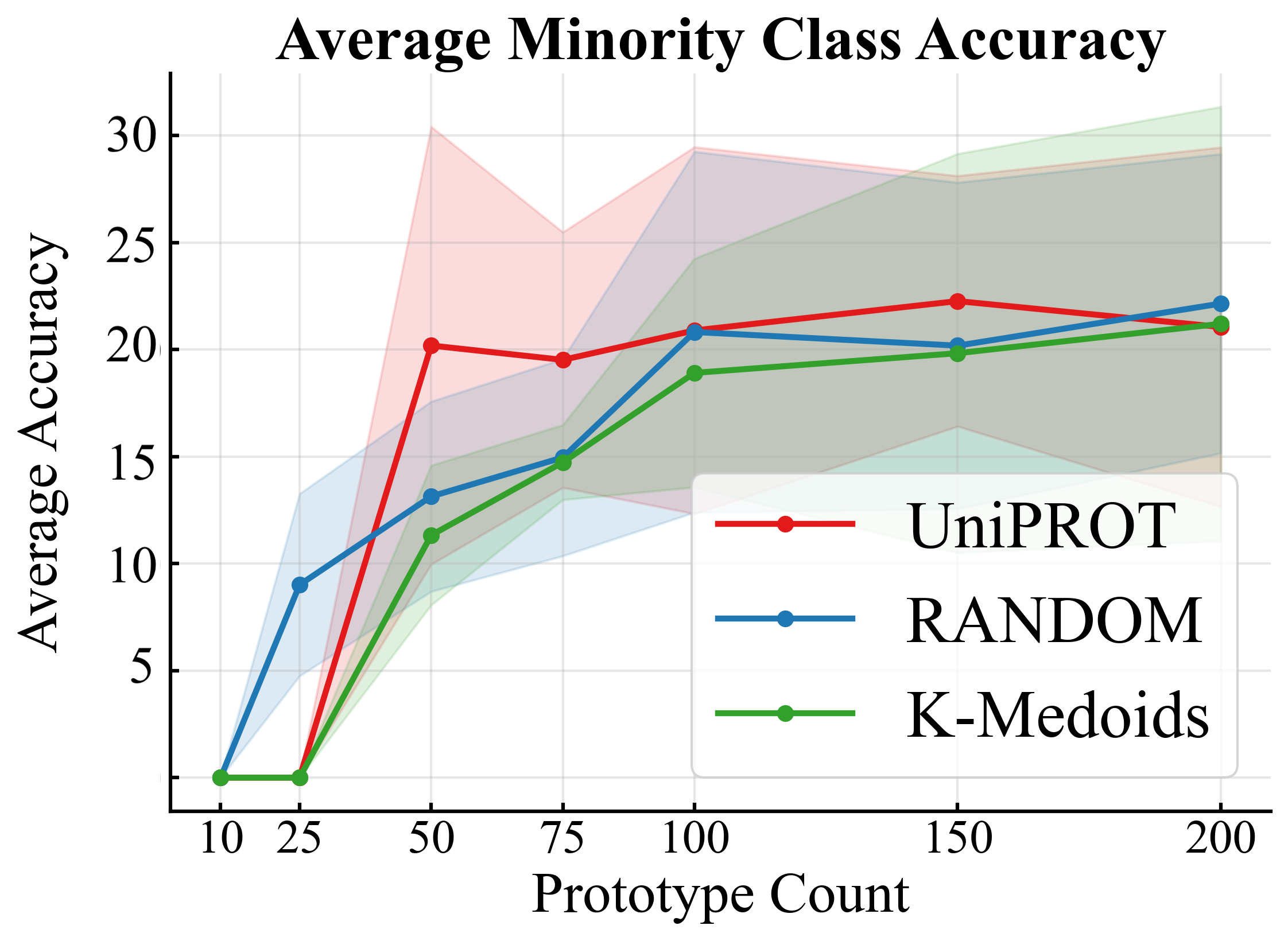}%
        \hfill
        \includegraphics[width=0.48\linewidth]{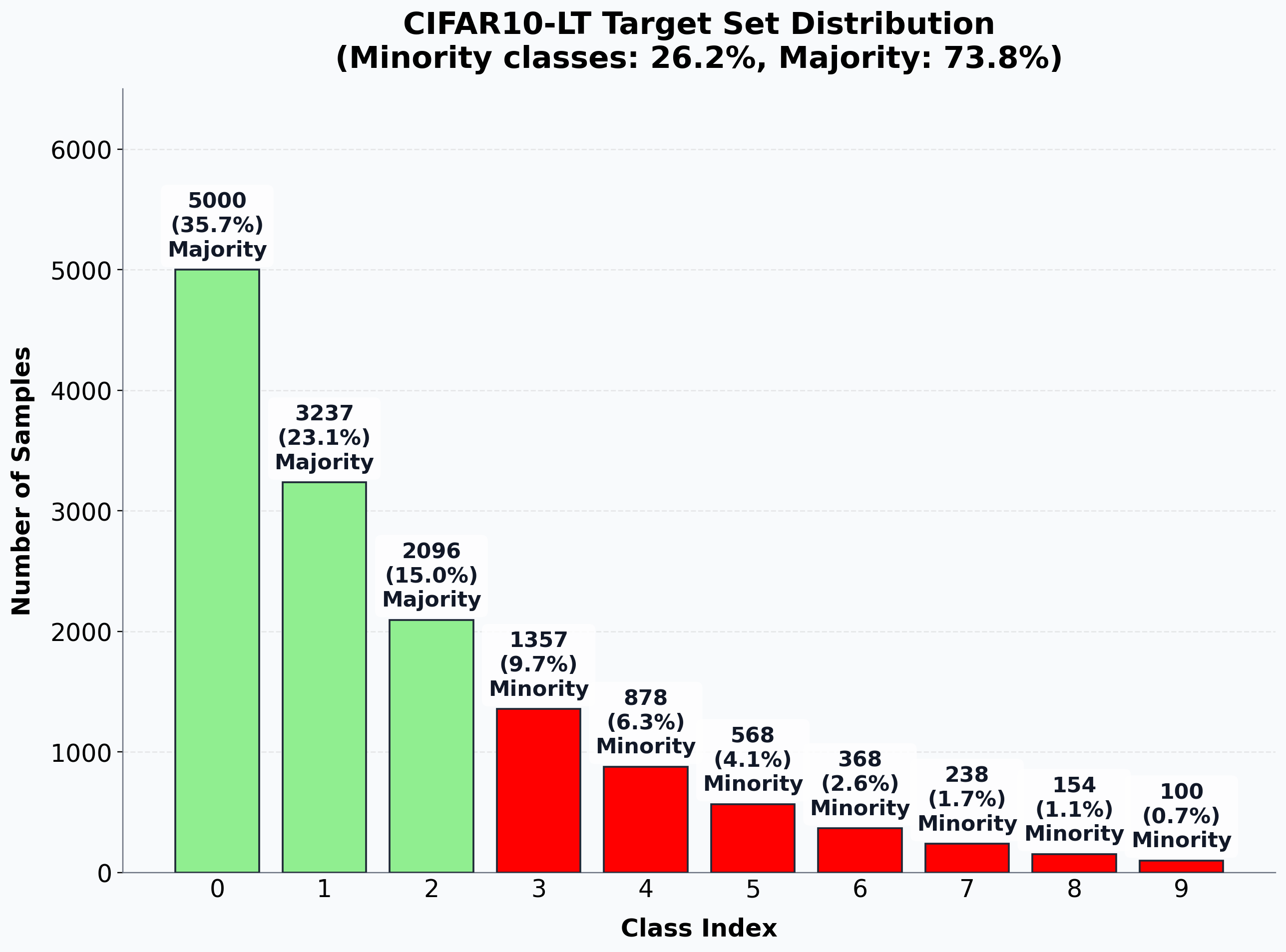} \\[2mm]
        
        \includegraphics[width=0.9\linewidth]{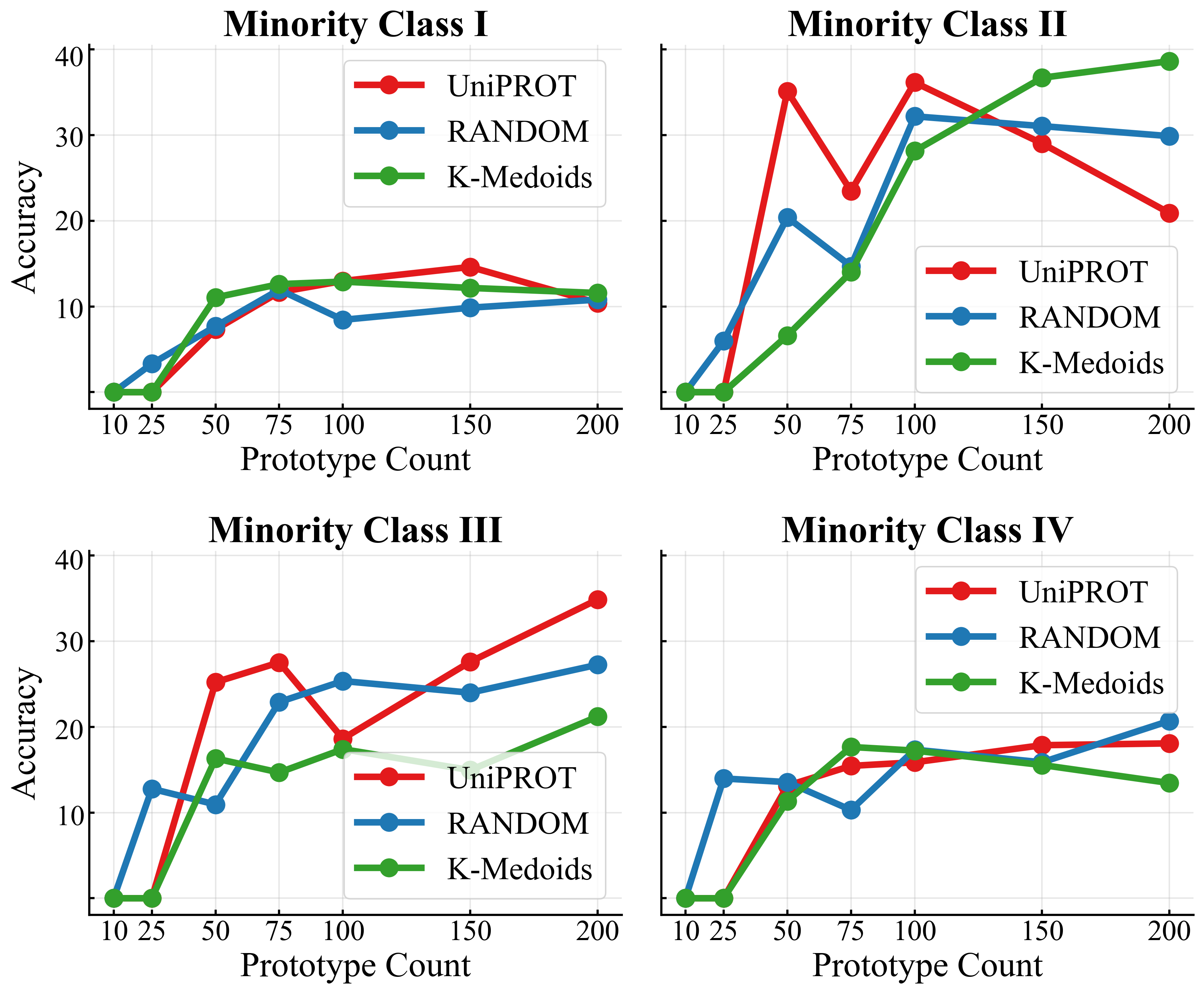}

        \caption{CIFAR10-LT}
    \end{subfigure}%
    \hfill
    \begin{subfigure}{0.48\textwidth}
        \centering
\includegraphics[height=0.4\linewidth,width=1\linewidth,keepaspectratio=false]{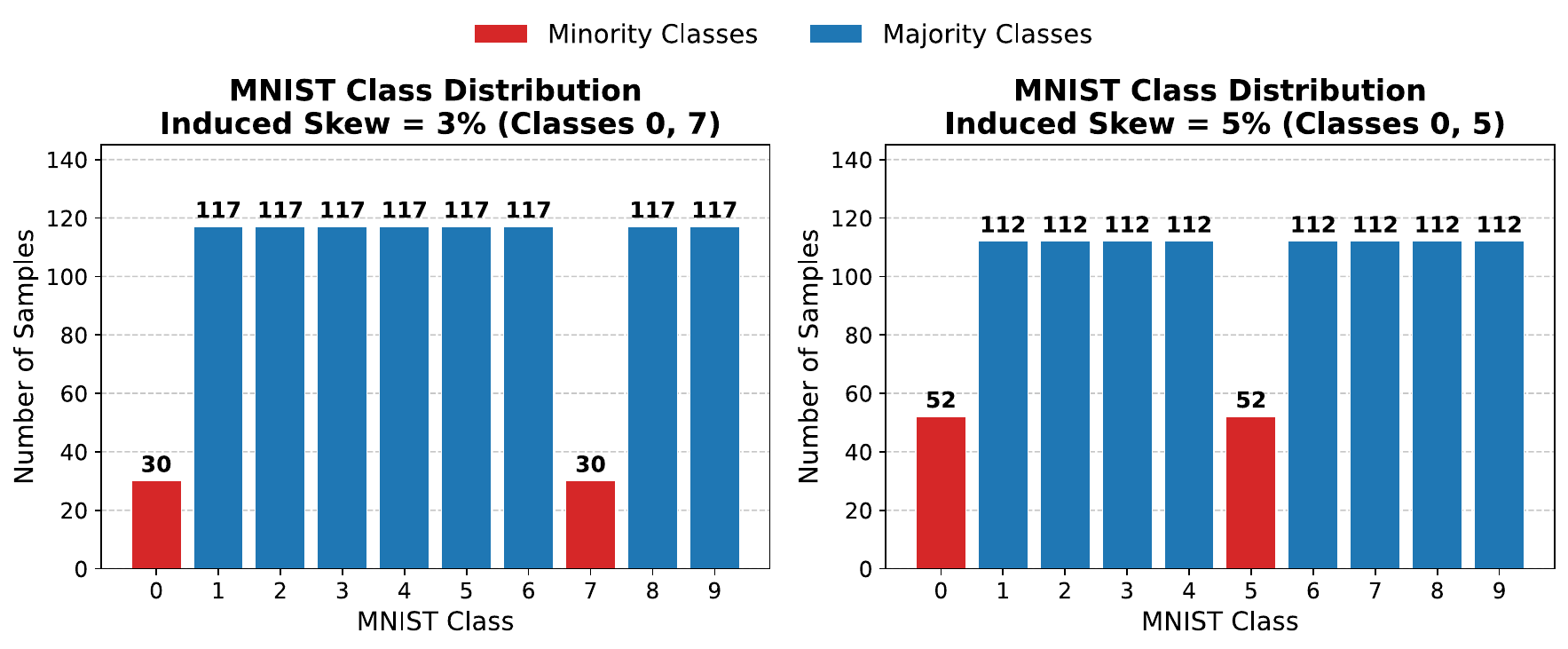} \\[2mm]
\includegraphics[height=0.7\linewidth,width=1\linewidth,keepaspectratio=false]{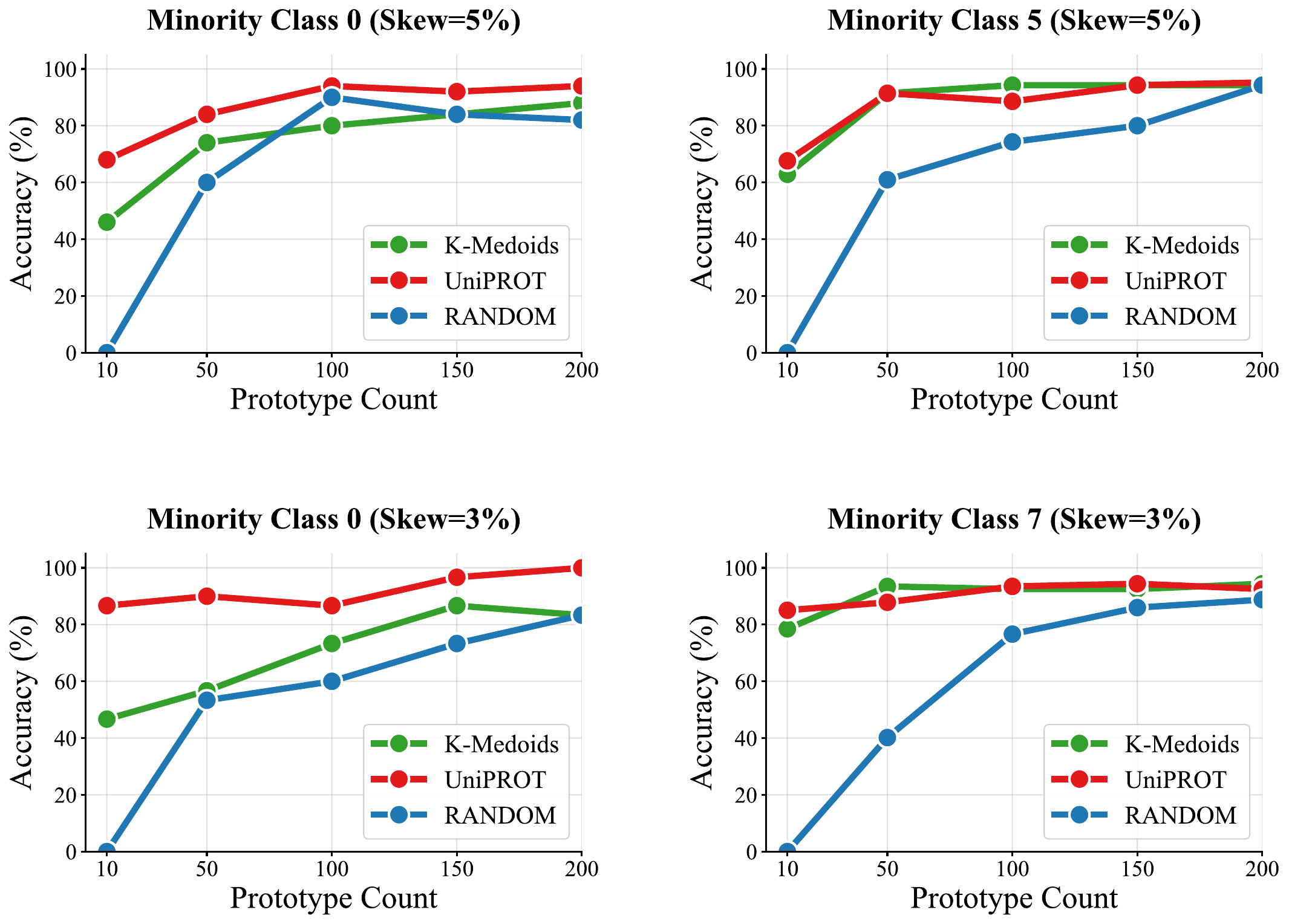}
        
        \caption{MNIST (Synthetic-LT)}
    \end{subfigure}

    \caption{\textit{Minority Class Accuracy Analysis}:On CIFAR10-LT and synthetic MNIST, $\methodprop$ outperforms $k$-medoids on minority classes.  
(\textbf{Top Left}): Average minority-class accuracy vs. prototype count.  
(\textbf{Bottom Left}): CIFAR10 minority class-wise accuracy (avg. over 3 runs).  
(\textbf{Top Right}): MNIST with Induced Skew: (i) classes 0 and 5 both at 5\% each (other classes at 90\%), and (ii) classes 0 and 7 both at 3\% each (other classes at 94\%). (\textbf{Bottom Right}): $\methodprop$ on average out-performs $k$-medoids across different skew variations for minority classes.
 }
    \label{fig:longtailclassifiation}
\end{figure*}

\subsection{High quality Mini-Batch Selection for LLM training}

Training LLMs with large mini-batches is known to accelerate convergence and improve model performance. However, this approach is often impractical due to the substantial memory  overheads. A common workaround is to select representative samples from a mini-batch that approximate the gradient of larger batches \citep{MirzasoleimanCraig20a,YangCrest23a}. Existing work \citep{killamsetty2021grad,killamsetty2021glister,wang2024greats,nguyen2024mini} rely on facility location or $k$-medoids based subset selection (\ref{eqn:spot}) to identify small high-quality mini-batches for LLM training. As discussed previously, these subsets approaches implicitly learn weighted representation. 
As the LLM training data is usually a highly imbalanced mixture of sources, weighted subset selection (\ref{eqn:spot}) may choose more representative prototypes with higher weights for larger sources and low-quality prototypes with low weights for smaller sources. 
However, this leads to misrepresentation of smaller sources and an eventual suboptimal performance of the training LLMs. To overcome this difficultly, we employ our uniformly weighted subset selection approach, $\methodprop$, in this problem setting and illustrate its suitability. 

\textbf{Problem Setting.} Consider a dataset $\V = \{\V_{1}\cup \cdots \cup \V_{p}\cup \V_{p+1} \cup \cdots \cup \V_{Q}\},$ with $Q$ sources, where the first $p$ sources correspond to minority sources and the remaining are majority sources. At iteration $t$, we have a (random) batch $\B_t$ which would typically have instances from all the sources. The aim is to select a highly representative subset of $\B_t$. One can perform this subset selection source-wise \citep{nguyen2024mini}, i.e., independently select $k_q$ instances from $\B^{t}_{q} = \B_t \cap \V_q\ \forall q$ where the budget $k_q$ is such that $\sum_q k_q = k$. Alternatively, one can directly select $k$ prototypes from $\B_t$. 

\textbf{Gradient based Representation for Subset Selection.} In order to select a subset of batch $\B_t$, we require a feature representation of the data points which is relevant to the subset selection problem. Recent works \citep{mirzasoleiman2020coresets,killamsetty2021grad,wang2024greats} have demonstrated the utility of the gradients as feature representation of the data points. 
Hence, during model training at iteration $t$, the similarity (or cost) matrix may be computed using the gradients of the data points in $\B_t$ (and also validation data points in case of \citep{wang2024greats}). However, as the dimensionality of gradients in LLMs is very large, employing exact gradients for subset selection problem may become impractical especially with high mini-batch size or low-resource hardware. Hence, \citep{nguyen2024mini} employed computationally efficient zeroth-order gradient approximation methods for constructing the similarity matrix for the subset selection problem. Overall, let the (gradient-based) representation of data points $\bx_i,\bx_j\in\B_t$ be $\bg^t(\bx_i)$  and $\bg^t(\bx_j)$, respectively, in the $t$-th iteration. Then, we employ the similarity matrix $\bS(i,j) = \inner{\bg^t(\bx_i),\bg^t(\bx_j)}$ in (\ref{eqn:proposed_3}) for our method $\methodprop$. On the other hand, \citep{nguyen2024mini} observed better results with  $\ell_1$ distance based similarity matrix, i.e., $\bS(i,j) = c - \|\bg^t(\bx_i)-\bg^t(\bx_j)\|_1$, where $c$ is a large constant which ensures the similarity matrix has positive entries. We also note that as \citep{wang2024greats} requires a validation dataset $U_t$, it computes both $\inner{\bg^t(\bx_i),\bg^t(\bx_j)}\ \forall \bx_i,\bx_j\in\B_t$ and $\inner{\bg^t(\bx_i),\bg^t(\bx_{\texttt{val}})}\ \forall \bx_i\in\B_t, \bx_{\texttt{val}}\in U_t$.

\begin{figure*}[htp!]
    \centering

    \small
    \setlength{\tabcolsep}{6pt}
    \renewcommand{\arraystretch}{1.2}
    \captionof{table}{Performance across in-domain and out-of-domain datasets for \textbf{\textsc{Phi-3}} on \textbf{\textsc{MathInstruct}}, batch size $|\B|=128$ and total budget $k=64$. For completeness and a \textbf{fair evaluation}, results are reported under two configurations for \textbf{all} baselines: \emph{source-wise}(left of ``/'') and \emph{batch-wise} (right of ``/''). $\methodprop$ consistently beats baselines in both configurations.}
    \label{table:main}
    \resizebox{0.98\textwidth}{!}{%
    \begin{tabular}{l c c c | c | c c c | c | c}
        \toprule
        \textbf{Method} & \multicolumn{4}{c}{\textbf{In-domain}} & \multicolumn{4}{c}{\textbf{Out-of-domain}} & \textbf{Avg-All} \\
        \cmidrule(lr){2-5} \cmidrule(lr){6-9}
        & \textbf{GSM8K} & \textbf{MATH} & \textbf{NumGLUE} & \textbf{Avg} 
        & \textbf{SVAMP} & \textbf{Mathematics} & \textbf{SimulEq} & \textbf{Avg} 
        &  \\
        \midrule
        FT & 76.72 & 36.54 & 62.57 & 58.61
            & 85.10 & 33.30 & 62.78 & 60.39
            & 59.50 \\
        MaxLoss & 70.64 \,/\, 69.44 & 32.05 \,/\, 30.02 & 57.80 \,/\, 58.90 & 53.50 \,/\, 52.79
            & 80.60 \,/\, 79.20 & 31.45 \,/\, 30.06 & 57.19 \,/\, 55.82 & 56.41 \,/\, 55.03
            & 54.96 \,/\, 53.91 \\
        GradNorm & 76.04 \,/\, 75.40 & 36.10 \,/\, 35.03 & 64.01 \,/\, 64.10 & 58.72 \,/\, 58.18
            & 85.30 \,/\, 84.17 & 38.00 \,/\, 36.50 & 61.84 \,/\, 65.70 & 61.71 \,/\, 62.12
            & 60.22 \,/\, 60.15 \\
        SBERT & 73.20 \,/\, 72.80 & 35.54 \,/\, 35.06 & 60.26 \,/\, 57.60 & 56.33 \,/\, 55.15
            & 79.31 \,/\, 77.90 & 34.05 \,/\, 34.00 & 61.70 \,/\, 58.90 & 58.35 \,/\, 56.93
            & 57.34 \,/\, 56.04 \\
        COLM & 76.80 \,/\, 76.36 & 37.28 \,/\, 36.42 & 64.11 \,/\, 64.10 & 59.40 \,/\, 58.96
            & 85.10 \,/\, 85.30 & \textbf{38.00} \,/\, 37.40 & 62.25 \,/\, 63.60 & 61.78 \,/\, 62.10
            & 60.59 \,/\, 60.53 \\
        GREATS & 76.72 \,/\, 77.80 & 37.84 \,/\, 37.28 & 67.46 \,/\, 64.40 & 60.67 \,/\, 59.83
            & 86.10 \,/\, 85.00 & 35.60 \,/\, \textbf{38.19} & 62.06 \,/\, 61.92 & 61.25 \,/\, 61.64
            & 60.96 \,/\, 60.73 \\
        \textbf{$\methodprop$} (Ours) & \textbf{79.07} \,/\, \textbf{78.16} & \textbf{38.40} \,/\, \textbf{37.76} & \textbf{68.80} \,/\, \textbf{66.02} & \textbf{62.09} \,/\, \textbf{60.65}
            & \textbf{86.20} \,/\, \textbf{85.70} & 36.90 \,/\, 37.20 & \textbf{66.73} \,/\, \textbf{68.28} & \textbf{63.28} \,/\, \textbf{63.73}
            & \textbf{62.68} \,/\, \textbf{62.19} \\
        \bottomrule
    \end{tabular}}

    \vspace{1em}

    \scalebox{0.30}{\input{iclr2026/plots/plot_main}}%
    \scalebox{0.30}{\input{iclr2026/plots/plot_zeph}}%
    \scalebox{0.30}{\input{iclr2026/plots/plot_phi2}}%

    \captionof{figure}{Validation perplexity dynamics on \textsc{Phi-3}, \textsc{Zephyr-3b} and \textsc{Phi-2} during training vs  top 3 best-performing baselines on \textbf{\textsc{MathInstruct}}. $\methodprop$-PS consistently outperforms other baselines.}
    \label{fig:main_combo}
\end{figure*}

\begin{figure*}[h!]
  \centering
  \begin{minipage}{0.55\textwidth}
    \centering
    \resizebox{\textwidth}{!}{%
      \input{iclr2026/plots/plot_bs64}%
      \input{iclr2026/plots/plot_bs32}%
      \input{iclr2026/plots/plot_bs16}%
    }
  \end{minipage}\hfill
  \begin{minipage}{0.42\textwidth}
    \centering
    \scriptsize
    \captionof{table}{Variation with $\lambda$ (entropic regularization).}
    \label{table:reg_ablation}
    \rowcolors{2}{gray!20}{white}
    \begin{tabular}{lccc}
      \toprule
      \rowcolor{gray!20}\textbf{$\lambda$ (Entropic Reg.)} & \textbf{GSM-8k} & \textbf{NumGlue} & \textbf{Svamp} \\
      \midrule
      0.1  & 47.76 & 37.5 & 52.9 \\
      0.05 & 48.36 & 36.1 & 54   \\
      0.01 & 49.40 & 36.7 & 54.5 \\
      \bottomrule
    \end{tabular}
  \end{minipage}
  \caption{Validation log-pplx with changing prototype percentage (left) and ablation table (right) on \textsc{Zephyr-3b}.}
  \label{fig:subsetratio_results}
\end{figure*}

\textbf{$\methodprop$-PS (Per Source)}. Given each batch consists of samples drawn from multiple sources, the objective here is to perform prototype selection at a per-source level. In particular, \citep{nguyen2024mini} ensures that samples belonging to minority sources are also preserved in the final selection, while prototype selection is done only for majority sources, thus preventing less representation of minority sources in the final selected subset. We consider $\methodprop$ at each source level per batch with each source having some cardinality constraint.
    The POT problem can then be formulated as $ f(\P_q) \coloneqq  \maxop_{\mgamma \in \Gamma_{\leq}(\bone_{\P_q},k_q\bone/n)}\inner{\bS_q,\mgamma}$  where $\P_q$ indicates the prototypes per $q$-th data source to be selected and $k_q$ indicates the cardinality per source and $n_q$ indicates the total size of the samples per source in the batch $\B_t$ (i.e., $n_q = |\B^t_q|$). Hence, the selected subset would be $\bigcup_{q=1}\P_q$. Here, we define $\bS \in \RR^{|\B^t_q| \times |\B^t_q|}$ where both the source and target are same $\S = \T = \B^t_q$
 ($q$th data source in the batch).

\textbf{$\methodprop$-PB (Per Batch)}. Beyond per-source prototype selection, we also consider the joint selection of prototypes across the entire batch, i.e., over all samples aggregated from multiple sources. This broader perspective is particularly beneficial, as $\methodprop$ mitigates the risk of systematically down-weighting minority sources, a bias that methods such as $k$-medoids or Facility Location clustering may inadvertently introduce as we observe in Figure \ref{fig:longtailclassifiation}. Here, the similarity matrix $\bS \in \RR^{|\B_t|\times |\B_t|}$ is formed in all samples within the entire batch.


\textbf{Models and Training details.} We evaluate on \textsc{Phi-2} (2.7B) \citep{javaheripi2023phi}, \textsc{Phi-3}(3.8B) \citep{li2023textbooks}, and \textsc{Zephyr}  (3B)
\citep{tunstall2023zephyr}. 
 For finetuning, we employ LoRA with rank $128$, $\alpha=512$, and dropout $0.05$. Following \citep{nguyen2024mini}, the LoRA are applied to all attention matrices (\textsc{QKV proj}) and two fully connected layers for the \textsc{Phi} models; for \textsc{Zephyr}, adapters are applied to all attention matrices (\textsc{QKVO proj}). All experiments are conducted on 3×A6000 GPUs. 

\textbf{Baselines.} We compare $\methodprop$ against (i) standard fine-tuning (FT), (ii) recent mini-batch selection approaches such as GREATS \citep{wang2024greats} and COLM \citep{nguyen2024mini}, and (iii) one-shot selection strategies such as Grad Norm (GN) \citep{katharopoulos2018not} and MaxLoss \citep{shalev2016minimizing}, adapted to mini-batch selection setting. For completeness and a \textbf{fair evaluation}, we include baseline results under both source-wise (\textbf{PS}) and batch-wise (\textbf{PB}) settings, whenever relevant. These are defined analogously to $\methodprop$-PS and $\methodprop$-PB. For the Full-Finetuning (\textbf{FT}) baseline, the model is trained as usual without any prototype selection. We note that \citep{nguyen2024mini} employs $k$-medoids for subset selection and demonstrate source-wise selection being better than batch-wise selection. On the other hand, \citep{wang2024greats} assumes access to a validation set and their subset selection criterion could be viewed as optimizing maximum mean discrepancy between the prototypical set (a subset of $\B_t$) and the validation set (with gradient based features and linear kernels). For a fair comparison, we simply assume random subset of training set as "validation anchors" for \citep{wang2024greats}.

\textbf{Finetuning Datasets.} We train on the \textsc{MathInstruct} dataset \citep{yue2023mammoth}, which consists of 260K instruction tuning samples curated from 14 highly imbalanced data sources. In addition, on the \textsc{SuperGLUE} benchmark \citep{wang2019superglue} for the following classification tasks \textit{SST-2}, \textit{CB}, and \textit{MultiRC}. We note that \textsc{SuperGLUE} does not have source information.

\textbf{Evaluation datasets.} Following \citep{yue2023mammoth}, we evaluate the finetuned models on both in-domain and out-of-domain benchmarks. The in-domain set comprises GSM8K \citep{cobbe2021training}, MATH \citep{hendrycks2021measuring}, and NumGLUE \citep{mishra2022numglue}, whereas the out-of-domain set includes SVAMP \citep{patel2021nlp}, Mathematics \citep{davies2021advancing}, and SimulEq \citep{koncel2016mawps}.



\subsection{Evaluation Results and Discussion}

\textbf{Finetuning Experiments:} We train all models for 2048 steps with batch size of 128 and prototype-ratio as 50\% (thus effective batch size $k$=64), except for Full-Finetuning (\textbf{FT}). The Full-Finetuning baselines train on the whole input-batch as is standard. Table \ref{tab:superglue} presents the results on batch selection where we train all baselines using batch selection for a fair comparison and $\methodprop$ performs well in full-batch selection. 
 Table \ref{table:main} shows the results of source-wise finetuning on \textsc{MathInstruct}. We do source selection on \textbf{all} baselines for fair comparision. The table indicates that $\methodprop$ is significantly better than other baselines in both the in-domain and out-of-domain settings. \\ \\
 \textbf{Source unavailable datasets:} In some benchmarks, such as \textsc{SuperGLUE}, source annotations are not provided, making source-wise selection infeasible. This setting is common in practice, where fine-tuning data may arrive without clear domain tags. To test robustness, we evaluate on SST2, MultiRC, and CB under this source-agnostic regime. Since ground-truth sources are unavailable, $\methodprop$-PS is omitted and all baselines are compared via batch selection for fairness. We train for 512 steps, batch-size 32 and selection ratio of 25\%, except for FT where we train on whole batch. As shown in Table~\ref{tab:superglue}, $\methodprop$ still outperforms alternatives, indicating strong performance without source information. 


\textbf{Pretraining Experiments}
    To test the effectiveness of $\methodprop$, we conduct pretraining experiments with OpenWebText on LLaMA-3 500M and 60M models for 20k steps. All baselines are trained in \textit{batch-wise}(PB) mode. We defer the details to Appendix \ref{supp:pretraining}. Figure \ref{fig:pretraining} indicates that $\methodprop$ outperforms all baselines including Full-batch pretraining both at large and small scale models.

\begin{figure}[htbp]
 \centering
 \scalebox{0.26}{\input{iclr2026/plots/plot_pret}}
 \scalebox{0.26}{\input{iclr2026/plots/plot_pret_60m}}
 \caption{Perplexity dynamics on \textbf{Pretraining} \textsc{LLaMA-3} 500M and 60M for 20k steps.}
 \label{fig:pretraining}
\end{figure}
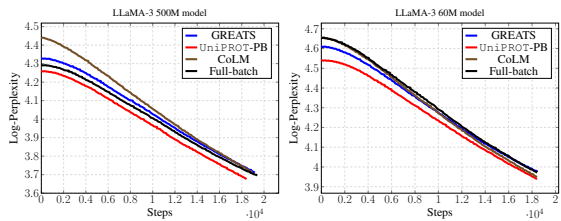
\begin{table}[t] 
    \centering
    \scriptsize
    \caption{Comparison of performance of batch-wise selection across baselines for \textbf{\textsc{Phi-3}} in \textbf{\textsc{SuperGlue}} \citep{wang2019superglue}.} 
    \label{tab:superglue}
    \begin{tabular}{l c c c | c}
        \toprule
        \textbf{Method} & \textbf{SST2} & \textbf{MultiRC} & \textbf{CB} & \textbf{Avg} \\
        \midrule
        FT  & 93.91 & 86.05 & 92.72 & 90.89 \\
        GradNorm   & 87.94 & 57.54 & 69.10 & 71.53 \\   
        SBERT   & 90.10 & 82.11 & 87.27 & 86.49 \\
        CoLM   & 94.72 & 82.99 & 93.05 & 90.25 \\
        GREATS   & \textbf{94.81} & \textbf{88.42} & 93.12 & 92.12 \\
        \textbf{$\methodprop$-PB (Ours)}  & 94.65 & 88.03 & {\textbf{94.54}} & \textbf{92.41} \\
        \bottomrule
    \end{tabular}
\end{table}

\textbf{Effect of number of selected prototypes}. We finetune \textsc{Zepyr-3b} on \textsc{MathInstruct} for 2048 steps while varying the selection ratio $r \in {50\%, 25\%, 12.5\%}$. 
Figure~\ref{fig:subsetratio_results} reports the validation perplexity. We observe that $\methodprop$ remains consistently stable across all ratios. In contrast, \textsc{CoLM} degrades noticeably as $r$ decreases, while \textsc{GREATS} shows a smaller but still measurable rise. Overall, $\methodprop$ exhibits the lowest sensitivity to prototype budget, indicating stronger robustness.

\textbf{Effect of regularization on downstream performance}.
We ablate the entropic regularization coefficient in the partial optimal transport objective by finetuning \textsc{Phi-3} on \textsc{MathInstruct} and evaluating downstream on GSM8K, NumGLUE and Svamp. We finetune for 128 steps and 20 OT iterations.(Table~\ref{table:reg_ablation}). With $\lambda=0.1$, performance is consistently lower, while reducing to $\lambda=0.01$ yields clear gains on both tasks. This trend aligns with prior observations \citep{cuturi2013lightspeed}, where smaller regularization improves transport fidelity and leads to better downstream accuracy.

\section{Conclusion} 
We proposed {\methodprop}, a scalable and theoretically grounded framework for selecting $k$ uniformly weighted prototypes that summarize a target distribution via optimal transport. This leads to a super-additive maximization problem under cardinality constraints, for which we introduced a novel submodular reformulation with a provably tight equivalence at $k$, enabling a greedy algorithm with a $(1 - 1/e)$ approximation guarantee. $\methodprop $ consistently improves minority-class representation in imbalanced classification tasks and enhances mini-batch quality for training large language models, outperforming existing methods in both accuracy and efficiency. By imposing uniform weights, $\methodprop $ mitigates bias toward majority classes, and supporting equitable learning.

\section{Acknowledgements}
PC acknowledges the Microsoft Research India PhD Award and Prime Minister Research Fellowship to support this research work. GR thanks Bank of Baroda Chair Professorship. PJ acknowledges the support of Anusandhan National Research Foundation under ARG-MATRICS program and IIT Bombay seed grant. 

\bibliography{AISTATS2026PaperPack/reference}
\bibliographystyle{apalike}

\newpage

\section*{Checklist}

The checklist follows the references. For each question, choose your answer from the three possible options: Yes, No, Not Applicable.  You are encouraged to include a justification to your answer, either by referencing the appropriate section of your paper or providing a brief inline description (1-2 sentences). 
Please do not modify the questions.  Note that the Checklist section does not count towards the page limit. Not including the checklist in the first submission won't result in desk rejection, although in such case we will ask you to upload it during the author response period and include it in camera ready (if accepted).

\textbf{In your paper, please delete this instructions block and only keep the Checklist section heading above along with the questions/answers below.}

\begin{enumerate}

  \item For all models and algorithms presented, check if you include:
  \begin{enumerate}
    \item A clear description of the mathematical setting, assumptions, algorithm, and/or model.  [\textcolor{blue}{Yes}]
    \item An analysis of the properties and complexity (time, space, sample size) of any algorithm. [\textcolor{blue}{Yes}]
    \item (Optional) Anonymized source code, with specification of all dependencies, including external libraries. [\textcolor{blue}{Yes}]
  \end{enumerate}

  \item For any theoretical claim, check if you include:
  \begin{enumerate}
    \item Statements of the full set of assumptions of all theoretical results. [\textcolor{blue}{Yes}]
    \item Complete proofs of all theoretical results. [\textcolor{blue}{Yes}]
    \item Clear explanations of any assumptions. [\textcolor{blue}{Yes}]     
  \end{enumerate}

  \item For all figures and tables that present empirical results, check if you include:
  \begin{enumerate}
    \item The code, data, and instructions needed to reproduce the main experimental results (either in the supplemental material or as a URL). [\textcolor{blue}{Yes}]
    \item All the training details (e.g., data splits, hyperparameters, how they were chosen). [\textcolor{blue}{Yes}]
    \item A clear definition of the specific measure or statistics and error bars (e.g., with respect to the random seed after running experiments multiple times). [\textcolor{blue}{Yes}]
    \item A description of the computing infrastructure used. (e.g., type of GPUs, internal cluster, or cloud provider). [\textcolor{blue}{Yes}]
  \end{enumerate}

  \item If you are using existing assets (e.g., code, data, models) or curating/releasing new assets, check if you include:
  \begin{enumerate}
    \item Citations of the creator If your work uses existing assets. [\textcolor{blue}{Yes}]
    \item The license information of the assets, if applicable. [\textcolor{blue}{Not Applicable}]
    \item New assets either in the supplemental material or as a URL, if applicable. [\textcolor{blue}{Not Applicable}]
    \item Information about consent from data providers/curators. [\textcolor{blue}{Not Applicable}]
    \item Discussion of sensible content if applicable, e.g., personally identifiable information or offensive content. [\textcolor{blue}{Not Applicable}]
  \end{enumerate}

  \item If you used crowdsourcing or conducted research with human subjects, check if you include:
  \begin{enumerate}
    \item The full text of instructions given to participants and screenshots. [\textcolor{blue}{Not Applicable}]
    \item Descriptions of potential participant risks, with links to Institutional Review Board (IRB) approvals if applicable. [\textcolor{blue}{Not Applicable}]
    \item The estimated hourly wage paid to participants and the total amount spent on participant compensation. [\textcolor{blue}{Not Applicable}]
  \end{enumerate}

\end{enumerate}

    \addtocontents{toc}{\protect
    \setcounter{tocdepth}{2}} 

    \onecolumn \par
    \noindent

    \par
    \noindent
    \rule{\textwidth}{3pt}
    \begin{center}
        \large\textbf{Supplementary Material: \papertitle}
    \end{center}
    \par
    \noindent
    \rule{\textwidth}{0.4pt}

    \appendix
    \addappheadtotoc
    \tableofcontents

    \renewcommand{\thesection}{\Alph{section}} 
    \renewcommand{\thesubsection}{\Alph{section}.\arabic{subsection}} 

\input{supplemental}
    
\end{document}


%
\runningtitle{\papertitle}

%

\onecolumn
\aistatstitle{Instructions for Paper Submissions to AISTATS 2026: \\
Supplementary Materials}

\section{FORMATTING INSTRUCTIONS}

To prepare a supplementary pdf file, we ask the authors to use \texttt{aistats2026.sty} as a style file and to follow the same formatting instructions as in the main paper.
The only difference is that the supplementary material must be in a \emph{single-column} format.
You can use \texttt{supplement.tex} in our starter pack as a starting point, or append the supplementary content to the main paper and split the final PDF into two separate files.

Note that reviewers are under no obligation to examine your supplementary material.

\section{MISSING PROOFS}

The supplementary materials may contain detailed proofs of the results that are missing in the main paper.

\subsection{Proof of Lemma 3}

\textit{In this section, we present the detailed proof of Lemma 3 and then [ ... ]}

\section{ADDITIONAL EXPERIMENTS}

If you have additional experimental results, you may include them in the supplementary materials.

\subsection{Effect of the Regularization Parameter}

\textit{Our algorithm depends on the regularization parameter $\lambda$. Figure 1 below illustrates the effect of this parameter on the performance of our algorithm. As we can see, [ ... ]}

\vfill

%% file: AISTATS2026PaperPack/math_commands.tex
\usepackage{amsmath,amsfonts,bm}
\usepackage{amsthm}

\def\eqref#1{equation~\ref{#1}}

\def\1{\bm{1}}

\DeclareMathAlphabet{\mathsfit}{\encodingdefault}{\sfdefault}{m}{sl}
\SetMathAlphabet{\mathsfit}{bold}{\encodingdefault}{\sfdefault}{bx}{n}

\def\sT{{\mathbb{T}}}

\newcommand{\R}{\mathbb{R}}

\DeclareMathOperator*{\argmax}{arg\,max}

\newtheorem*{restatelemma}{Restated Lemma}

%% file: macros.tex
\def \papertitle{Uniform Prototype Selection via Partial Optimal Transport with Submodular Guarantees}

\newcommand{\bD}[1]{\boldsymbol{\mathcal{D}}}

\newcommand{\methodprop}{\texttt{UniPROT}}

\newcommand{\plotwidth}{3pt}

\DeclareMathAlphabet{\pazocal}{OMS}{zplm}{m}{n}

\def\BibTeX{{\rm B\kern-.05em{\sc i\kern-.025em b}\kern-.08em
    T\kern-.1667em\lower.7ex\hbox{E}\kern-.125emX}}

\newcommand{\inner}[1]{\left\langle#1\right\rangle}
\def\H{\mathcal{H}}

\def\G{\mathcal{G}}

\def\A{\mathcal{A}}
\def\X{\mathcal{X}}
\def\Y{\mathcal{Y}}
\def\RR{\mathbb{R}}

\def\B{\mathcal{B}}
\def\U{\mathcal{U}}
\def\P{\mathcal{P}}

\def\bzero{{\mathbf 0}}
\def\bone{{\mathbf 1}}

\def\bg{{\mathbf g}}

\def\bv{\mathbf v}

\def\bx{{\boldsymbol x}}
\def\by{{\mathbf y}}
\def\bz{{\boldsymbol z}}

\def\bC{\mathbf C}
\def\bD{\mathbf D}

\def\bS{\mathbf S}
\def\bT{\mathbf T}

\def\bW{{\mathbf W}}

\def\bZ{{\mathbf Z}}

\def\bmu{{\boldsymbol \mu}}
\def\bnu{{\boldsymbol \nu}}

\def\btheta{\boldsymbol{\theta}}

\def\argmax{\mathop{\rm arg\,max}\limits}

\def\maxop{\mathop{\rm max}\limits} 
\def\minop{\mathop{\rm min}\limits}

\newcommand{\eat}[1]{}

\def\H{\mathcal{H}}
\def\I{\mathcal{I}}

\def\R{\mathcal{R}}
\def\S{\mathcal{S}}
\def\T{\mathcal{T}}
\def\V{\mathcal{V}}

\usepackage{amsmath,amsthm,amssymb,amsxtra,mathtools,bm}
\usepackage{array}
\usepackage{graphicx}
\usepackage{physics}
\usepackage{MnSymbol,wasysym}
\usepackage{tikzsymbols}
\usepackage{array,tabularx,multirow}
\usepackage{makecell}
\usepackage{array}

\usepackage{array}
\newcolumntype{P}[1]{>{\centering\arraybackslash}p{#1}}

\usepackage{verbatim}
\usepackage{enumerate}
\usepackage{parskip}

\usepackage{color,soul}
\definecolor{clemson-orange}{RGB}{234,106,32}
\definecolor{highlight-orange}{RGB}{255,150,150}
\definecolor{chicago-maroon}{RGB}{128,0,0}
\definecolor{cincinnati-red}{RGB}{190,0,0}
\definecolor{soft-cyan}{RGB}{68,85,90}
\definecolor{firebrick}{RGB}{178,34,34}
\definecolor{crimson}{RGB}{220,20,60}
\definecolor{cerrulean}{rgb}{0.165,0.322,0.745}
\definecolor{jaam}{rgb}{0.45,0.0,0.45}

\usepackage{hyperref}
\hypersetup{
  colorlinks   = true,    
  urlcolor     = jaam,    
  linkcolor    = firebrick,    
  citecolor    = blue      
}

\usepackage{parskip}
\usepackage{thmtools}
\declaretheoremstyle[
   headfont=\bfseries, 
   notebraces={[}{]},
   bodyfont=\normalfont\itshape, spaceabove=10pt,
   spacebelow=10pt]{mystyle}
 \theoremstyle{mystyle}
\theoremstyle{definition}

\newif\ifsolutions \solutionstrue

\def\final{0}
\newcommand{\reviewer}[3]{
  \expandafter\newcommand\csname #1\endcsname[1]{
    \ifthenelse{\equal{\final}{1}} {
      \textcolor{#3}{}
    } {
      \textcolor{#3}{\begin{center} \textbf{#2} ##1 \end{center}}
    }
  }
}

\def\1{\bm{1}}

\def\valpha{{\bm{\alpha}}}

\def\mgamma{{\bm{\gamma}}}

\DeclareMathAlphabet{\mathsfit}{\encodingdefault}{\sfdefault}{m}{sl}
\SetMathAlphabet{\mathsfit}{bold}{\encodingdefault}{\sfdefault}{bx}{n}

\def\sT{{\mathbb{T}}}

\renewcommand{\H}{{\ve H}}

\newtheorem{thm}{Theorem}
\newtheorem{thmnew}{Theorem}

\newtheorem{lem}[thmnew]{Lemma}

\newtheorem{rem}{Remark}

\DeclareMathAlphabet{\pazocal}{OMS}{zplm}{m}{n}

\def\BibTeX{{\rm B\kern-.05em{\sc i\kern-.025em b}\kern-.08em
    T\kern-.1667em\lower.7ex\hbox{E}\kern-.125emX}}

\def\H{\mathcal{H}}

\def\G{\mathcal{G}}

\def\A{\mathcal{A}}
\def\X{\mathcal{X}}
\def\Y{\mathcal{Y}}
\def\RR{\mathbb{R}}

\def\B{\mathcal{B}}
\def\U{\mathcal{U}}
\def\Q{\mathcal{Q}}
\def\P{\mathcal{P}}

\def\bzero{{\mathbf 0}}
\def\bone{{\mathbf 1}}

\def\bg{{\mathbf g}}

\def\bv{\mathbf v}

\def\bx{{\mathbf x}}
\def\by{{\mathbf y}}
\def\bz{{\mathbf z}}

\def\bC{\mathbf C}
\def\bD{\mathbf D}

\def\bS{\mathbf S}
\def\bT{\mathbf T}

\def\bW{{\mathbf W}}

\def\bZ{{\mathbf Z}}

\def\bmu{{\boldsymbol \mu}}
\def\bnu{{\boldsymbol \nu}}

\def\argmax{\mathop{\rm arg\,max}\limits}

\def\maxop{\mathop{\rm max}\limits} 
\def\minop{\mathop{\rm min}\limits}

\def\H{\mathcal{H}}
\def\I{\mathcal{I}}

\def\R{\mathcal{R}}
\def\S{\mathcal{S}}
\def\T{\mathcal{T}}
\def\V{\mathcal{V}}

%% file: iclr2026/plots/plot_main.tex
\begin{tikzpicture}
    \begin{axis}[
        width=0.8\textwidth,
        height=0.6\textwidth,
        xlabel={Steps},
        ylabel={Log-Perplexity},
         title={\LARGE \textbf{Phi-3}},
        grid=both,
        grid style={dashed, gray!30},
        legend pos=north east,
        xmin = 0,
        clip=false,
        font=\LARGE, 
        label style={font=\LARGE}, 
        tick label style={font=\LARGE, rotate=45}, 
        title style={font=\LARGE}, 
        legend style={font=\LARGE} 
    ]
        \addplot+[mark=none, line width=\plotwidth] table [x index=0, y index=1] {iclr2026/data/greats_main.dat};
        \addlegendentry{GREATS}
        \addplot+[mark=none, line width=\plotwidth] table [x index=0, y index=1] {iclr2026/data/uniprot_main.dat};
        \addlegendentry{$\methodprop$-PS}
        \addplot+[mark=none, line width=\plotwidth] table [x index=0, y index=1] {iclr2026/data/colm_main.dat};
        \addlegendentry{CoLM}

     \end{axis}
\end{tikzpicture}

%% file: iclr2026/plots/plot_zeph.tex
\begin{tikzpicture}
    \begin{axis}[
        width=0.8\textwidth,
        height=0.6\textwidth,
        xlabel={Steps},
        ylabel={Log-Perplexity},
         title={\LARGE \textbf{Zephyr-3b}},
        grid=both,
        grid style={dashed, gray!30},
        legend pos=north east,
        xmin = 0,
        font=\LARGE, 
        label style={font=\LARGE}, 
        tick label style={font=\LARGE, rotate=45}, 
        title style={font=\LARGE}, 
        legend style={font=\LARGE} 
    ]
        \addplot+[mark=none, line width=\plotwidth] table [x index=0, y index=1] {iclr2026/data/greats_zeph.dat};
        \addlegendentry{GREATS}
        \addplot+[mark=none, line width=\plotwidth] table [x index=0, y index=1] {iclr2026/data/uniprot_zeph.dat};
        \addlegendentry{$\methodprop$-PS}
        \addplot+[mark=none, line width=\plotwidth] table [x index=0, y index=1] {iclr2026/data/colm_zeph.dat};
        \addlegendentry{CoLM}
    \end{axis}
\end{tikzpicture}

%% file: iclr2026/plots/plot_phi2.tex
\begin{tikzpicture}
    \begin{axis}[
        width=0.8\textwidth,
        height=0.6\textwidth,
        xlabel={Steps},
        ylabel={Log-Perplexity},
         title={\LARGE \textbf{Phi-2}},
        grid=both,
        xmin = 0,
        grid style={dashed, gray!30},
        legend pos=north east,
        font=\LARGE, 
        label style={font=\LARGE}, 
        tick label style={font=\LARGE, rotate=45}, 
        title style={font=\LARGE}, 
        legend style={font=\LARGE} 
    ]
        \addplot+[mark=none, line width=\plotwidth] table [x index=0, y index=1] {iclr2026/plots/plot_bs16_greats_phi2.dat};
        \addlegendentry{GREATS}
        \addplot+[mark=none, line width=\plotwidth] table [x index=0, y index=1] {iclr2026/plots/plot_bs16_phi2_uniprot.tex};
        \addlegendentry{$\methodprop$-PS}
        \addplot+[mark=none, line width=\plotwidth] table [x index=0, y index=1] {iclr2026/plots/plot_bs16_phi2_colm.tex};
        \addlegendentry{CoLM}
        
    \end{axis}
\end{tikzpicture}

%% file: iclr2026/plots/plot_bs64.tex
\begin{tikzpicture}
    \begin{axis}[
        width=0.8\textwidth,
        height=0.6\textwidth,
        xlabel={Steps},
        ylabel={Log-Perplexity},
         title={\LARGE \textbf{Subset ratio 50\% }},
        grid=both,
        grid style={dashed, gray!30},
        legend pos=north east,
        font=\LARGE, 
        label style={font=\LARGE}, 
        tick label style={font=\large}, 
        title style={font=\LARGE}, 
        legend style={font=\large} 
    ]
        \addplot+[mark=none, line width=\plotwidth] table [x index=0, y index=1] {iclr2026/data/greats_zeph.dat};
        \addlegendentry{GREATS}
        \addplot+[mark=none, line width=\plotwidth] table [x index=0, y index=1] {iclr2026/data/uniprot_zeph.dat};
        \addlegendentry{$\methodprop$-PS}
        \addplot+[mark=none, line width=\plotwidth] table [x index=0, y index=1] {iclr2026/data/colm_zeph.dat};
        \addlegendentry{CoLM}

        \addlegendentry{Ours}
    \end{axis}
\end{tikzpicture}

%% file: iclr2026/plots/plot_bs32.tex
\begin{tikzpicture}
    \begin{axis}[
        width=0.8\textwidth,
        height=0.6\textwidth,
        xlabel={Steps},
        ylabel={Log-Perplexity},
         title={\LARGE \textbf{Subset ratio 25\% }},
        grid=both,
        grid style={dashed, gray!30},
        legend pos=north east,
        font=\LARGE, 
        label style={font=\LARGE}, 
        tick label style={font=\large}, 
        title style={font=\LARGE}, 
        legend style={font=\large} 
    ]
        \addplot+[mark=none, line width=\plotwidth] table [x index=0, y index=1] {iclr2026/data/greats_bs32.dat};
        \addlegendentry{GREATS}
        \addplot+[mark=none, line width=\plotwidth] table [x index=0, y index=1] {iclr2026/data/uniprot_bs32.dat};
        \addlegendentry{$\methodprop$-PS}
        \addplot+[mark=none, line width=\plotwidth] table [x index=0, y index=1] {iclr2026/data/colm_bs32.dat};
        \addlegendentry{CoLM}
        
    \end{axis}
\end{tikzpicture}

%% file: iclr2026/plots/plot_bs16.tex
\begin{tikzpicture}
    \begin{axis}[
        width=0.8\textwidth,
        height=0.6\textwidth,
        xlabel={Steps},
        ylabel={Log-Perplexity},
         title={\LARGE \textbf{Subset ratio 12.5\% }},
        grid=both,
        grid style={dashed, gray!30},
        legend pos=north east,
        font=\LARGE, 
        label style={font=\LARGE}, 
        tick label style={font=\large}, 
        title style={font=\LARGE}, 
        legend style={font=\large} 
    ]
        \addplot+[mark=none, line width=\plotwidth] table [x index=0, y index=1] {iclr2026/data/greats_bs16.dat};
        \addlegendentry{GREATS}
        \addplot+[mark=none, line width=\plotwidth] table [x index=0, y index=1] {iclr2026/data/uniprot_bs16.dat};
        \addlegendentry{$\methodprop$-PS}
        \addplot+[mark=none, line width=\plotwidth] table [x index=0, y index=1] {iclr2026/data/colm_bs16.dat};
        \addlegendentry{CoLM}
        
    \end{axis}
\end{tikzpicture}

%% file: iclr2026/plots/plot_pret.tex
\begin{tikzpicture}
    \begin{axis}[
        width=0.8\textwidth,
        height=0.6\textwidth,
        xlabel={Steps},
        ylabel={Log-Perplexity},
         title={\Large LLaMA-3 500M model},
        grid=both,
        grid style={dashed, gray!30},
        legend pos=north east,
        xmin = 0,
        clip = true,
        font=\LARGE, 
        label style={font=\LARGE}, 
        tick label style={font=\LARGE}, 
        title style={font=\LARGE}, 
        legend style={font=\LARGE}, 
    ]
        \addplot+[mark=none, line width=\plotwidth] table [x index=0, y index=1] {iclr2026/data/greats_pret.dat};
        \addlegendentry{GREATS}
        \addplot+[mark=none, line width=\plotwidth] table [x index=0, y index=1] {iclr2026/data/uniprot_pret.dat};
        \addlegendentry{$\methodprop$-PB}
        \addplot+[mark=none, line width=\plotwidth] table [x index=0, y index=1] {iclr2026/data/colm_pret.dat};
        \addlegendentry{CoLM}
        \addplot+[mark=none, line width=\plotwidth] table [x index=0, y index=1] {iclr2026/data/full_pret.dat};
        \addlegendentry{Full-batch}
        
    \end{axis}
\end{tikzpicture}

%% file: iclr2026/plots/plot_pret_60m.tex
\begin{tikzpicture}
    \begin{axis}[
        width=0.8\textwidth,
        height=0.6\textwidth,
        xlabel={Steps},
        ylabel={Log-Perplexity},
        title={\Large LLaMA-3 60M model},
        grid=both,
        xmin=0,
        clip=true,
        grid style={dashed, gray!30},
        legend pos=north east,
        font=\LARGE, 
        label style={font=\LARGE}, 
        tick label style={font=\LARGE}, 
        title style={font=\LARGE}, 
        legend style={font=\LARGE} 
    ]
        \addplot+[mark=none, line width=\plotwidth] table [x index=0, y index=1] {iclr2026/data/60m_greats.dat};
        \addlegendentry{GREATS}
        \addplot+[mark=none, line width=\plotwidth] table [x index=0, y index=1] {iclr2026/data/60m_fairot.dat};
        \addlegendentry{$\methodprop$-PB}
        \addplot+[mark=none, line width=\plotwidth] table [x index=0, y index=1] {iclr2026/data/60m_colm.dat};
        \addlegendentry{CoLM}
        \addplot+[mark=none, line width=\plotwidth] table [x index=0, y index=1] {iclr2026/data/60m_full.dat};
        \addlegendentry{Full-batch}
    \end{axis}
\end{tikzpicture}

%% file: supplemental.tex
\newpage
\onecolumn
\allowdisplaybreaks
\par\noindent\rule{\textwidth}{1pt}
\begin{center}
\large\textbf{Supplementary Material: 
\papertitle}
\end{center}
\par\noindent\rule{\textwidth}{0.4pt}

\section{Organization of Appendix}
The Appendix is structured as follows. Section \ref{ref:TheoreticalResults} summarizes the main theoretical results. Section \ref{ref:AlgorithmDetails} presents the detailed description of the proposed algorithm. Sections \ref{ref:implementationDetails} and \ref{ref:ExperimentalSetupDetails} outline the implementation aspects and the specifics of the experimental setup, respectively. Finally, Section \ref{ref:BroaderImpact} discusses the broader impact of our work, and Section \ref{ref:Code} provides a link to the publicly available codebase.

\section{Theoretical Results}\label{ref:TheoreticalResults}

\textbf{Submodularity Ratio}\label{defn:submodularity ratio} 
The notion of submodularity ratio is given by approximate submodularity in \citep{das2018approximate}. For a monotone function $f$ the submodularity ratio w.r.t a set $S$ and a parameter $k\geq 0$
as \[
\alpha_{L,K}(f) = \min_{\substack{\S \subseteq L,\, \A \subseteq L \\ |\A| \leq K,\, \A \cap \S = \varnothing}} 
\frac{\sum_{u \in A} f(\S \cup \{u\}) - f(\S)}{f(\S \cup A) - f(\S)},
\quad \text{with } \frac{0}{0} := 1.
\]

$f$ is submodular if and only if \( \alpha_{L,K}(f) \geq 1 \). If the ratio
\[
\alpha := \frac{\sum_{u \in A} f(\S \cup \{u\}) - f(\S)}{f(\S \cup \A) - f(\S)}
\]
is strictly positive but not necessarily greater than 1, then \( f \) is said to be \( \alpha \)-weakly submodular.

\LemmaAdditivity*
\begin{proof}
We prove each property in turn (refer to the definition of $h(\P)$ in (\ref{eqn:proposed_2})). 

\noindent\textit{1. Non-negativity.}
The non-negativity follows from the definition of $h(\cdot)$ in (\ref{eqn:proposed_2}) namely, $h(\P) \coloneq \maxop_{\scriptscriptstyle \smash{\substack{\mgamma \in \Gamma(\bone_{\P},|\P|\bone/n)}}}
            \inner{\bS, \mgamma}$, where the similarity matrix $\bS$ is a non-negative matrix and the transport plan is enforced to  non-negative.

\noindent\textit{2. Monotonicity.}
Consider a subset $\P_1$ and define a set $\P_2 = \P_1 \cup \{\bx_i\}$ for any $\bx_i \notin \P_1$. To prove monotonicity of $h(.)$, it is sufficient to show that $h(\P_2) \geq h(\P_1)$. To this end, let $\mgamma_{\P_1}$ be the argmax for $h(\P_1)$ and consider the sub-matrix $\mgamma_{\P_1}\left(\I_{\P_1}, : \right)$ which is the restriction of the optimal solution to the points in $\P_1$. We note that $\mgamma_{\P_1}(i,:) = \bzero; \bx_i \notin \P_1$. We construct a feasible transport plan $\hat{\mgamma}$ for the set $\P_2$ as:
\begin{equation*}
    \hat{\mgamma}\left(\I_{\P_2}, : \right) = \left[\mgamma_{\P_1}\left(\I_{\P_1}, : \right)^\top, \frac{\bone}{n}\right]^\top,
\end{equation*}
and $\hat{\mgamma}(j,:) = \bzero$ for $\bx_j \notin \P_2$.
Let $\hat{h}\left(\P_2; \hat{\mgamma}\right)$ indicate the function value evaluated at the feasible transport plan $\hat{\mgamma}$ for the set $\P_2$. We then have
\begin{equation}
    \begin{aligned}
     h\left(\P_2\right) \geq \hat{h}\left(\P_2; \hat{\mgamma}\right) &= \inner{\bS\left(\I_{\P_2},:\right),  \hat{\mgamma}\left(\I_{\P_2}, : \right)} \\
     &= \inner{\bS\left(\I_{\P_1},:\right), \mgamma_{\P_1}\left(\I_{\P_1}, : \right)} + \inner{\bS(i,:), \frac{\bone}{n}}\\
     &= h(\P_1) + \inner{\bS(i,:), \frac{\bone}{n}}\\
     &\geq h(\P_1) \text{(Since } \bS(i,:)\geq \bzero\text{.)}
    \end{aligned}
\end{equation}

\noindent\textit{3. Super-additivity over disjoint sets.}
Consider two disjoint sets $\P_1$ and $\P_2$. Let $\mgamma_{\P_1}\left(\I_{\P_1}, : \right)$ and $\mgamma_{\P_2}\left(\I_{\P_2}, : \right)$ represent the sub-matrices of the respective optimal solutions to the points in $\P_1$ and $\P_2$. For the disjoint union set $\P = \P_1 \cup \P_2$, we construct a feasible transport plan $\hat{\mgamma}$ as:
\begin{equation*}
 \hat{\mgamma}\left(\I_{\P}, : \right) = \left[\mgamma_{\P_1}\left(\I_{\P_1}, : \right)^\top, \mgamma_{\P_2}\left(\I_{\P_2}, : \right)^\top\right]^\top,
\end{equation*}
and $\hat{\mgamma}(j,:) = \bzero$ for $\bx_j \notin \P$.
Evaluating the function $\hat{h}\left(\P;  \hat{\mgamma}\right)$ at the feasible solution, we get
\begin{equation}
\begin{aligned}
    h\left(\P\right) \geq \hat{h}\left(\P; \hat{\mgamma}\right) &= \inner{\bS\left(\I_{\P},:\right),  \hat{\mgamma}\left(\I_{\P}, : \right)}\\
    &= \inner{\bS\left(\I_{\P_1},:\right), \mgamma_{\P_1}\left(\I_{\P_1}, : \right)} + \inner{\bS\left(\I_{\P_2},:\right), \mgamma_{\P_2}\left(\I_{\P_2}, : \right)}\\
    &= h(\P_1) + h(\P_2).
\end{aligned}
\end{equation}

\end{proof}

\lemmaSubmod*
\begin{proof} We derive all the three properties below.

\noindent\textit{1. Non-negativity:} $f(\P)\geq 0\ \forall \P\subseteq \S$. The proof follows along similar lines of the non-negativity proof in Lemma~\ref{lemma:properties}.

\noindent\textit{2. Monotonicity:} $f(\P_{2}) \geq f(\P_{1})\ \forall \P_{1}\subseteq \P_{2}\subseteq\S$. Akin to the monotonicity proof in Lemma~\ref{lemma:properties}, for a super-set $\P_2 = \P_1 \cup \{\bx_i\}; \bx_i \notin \P_1$, we can construct a feasible transport $\hat{\mgamma}$ using the optimal solution $\mgamma_{\P_1}$ as:
\begin{equation*}
    \hat{\mgamma}\left(\I_{\P_2}, : \right) = \left[\mgamma_{\P_1}\left(\I_{\P_1}, : \right)^\top, \frac{\bv}{\|\bv\|_1} \right]^\top,
\end{equation*}
where $\bv = k\bone/n - \mgamma_{\P_1}^\top \bone \geq \bzero$. Following similar lines to the argument in Lemma~\ref{lemma:properties}, we obtain the monotonicity result.

\noindent\textit{3. Submodularity:}
To prove submodularity of function $f(\P)$ in (\ref{eqn:proposed_3}), we first note the following result \citep[Lemma 2]{Kawano22a}. 
\begin{restatable}{lemma}{Kawano}[\citep[Lemma 2]{Kawano22a}]
Let $l,m,n$ be positive integers. Given a positive valued $m\times n$ matrix $\bS>0$, the following set function $\psi:2^m\rightarrow \RR_{+}$ is a submodular function:
\begin{equation}\label{eqn:Kawano}
    \psi(\P) = \max_{\mgamma\in \Gamma_{\leq}(\bone_{\P}/l, \bone_n/n)} \inner{\bS,\mgamma},
\end{equation}
where as defined earlier, $\Gamma_{\leq}(\bmu, \bnu)=\{\mgamma \in \RR^{m \times n} \mid \mgamma \geq 0, \mgamma\bone = \bmu, \mgamma^{\top}\bone \leq \bnu\}$ and $\bone_n$ is a $n\times 1$ vector of $1$. 
\end{restatable}
We observe that for $l=k$, (\ref{eqn:Kawano}) is equivalent to the proposed function $f(\cdot)$ defined in (\ref{eqn:proposed_3}) as follows:
\begin{itemize}
    \item For a given $\P$, let $\mgamma_1$ be an optimal solution for (\ref{eqn:Kawano}). Then, $\mgamma_2 = k\mgamma_1$ is an optimal coupling for computing $f(\P)$ in (\ref{eqn:proposed_3}). Similarly, if $\mgamma_2$ be an optimal solution for computing $f(\P)$ in (\ref{eqn:proposed_3}), then $\mgamma_1 = \mgamma_2/k$ is an optimal solution for computing $f(\P)$ in (\ref{eqn:Kawano}).
    \item Hence, for a given $\P$, $f(\P) = k\psi(\P)$
\end{itemize}
Due to the above, $\forall A, B\subseteq \S$ 
$$\psi(A\cup B) + \psi(A \cap B) \leq \psi(A) + \psi(B) \Rightarrow f(A\cup B) + f(A \cap B) \leq f(A) + f(B),$$ 
which proves that $f$ is a submodular function.

\end{proof}
\equivalencelemma*
\begin{proof}
Recall that for any set $\P$ of cardinality $k$, $f(\P) = h(\P)$. Due the monotonicity properties in Lemmas~\ref{lemma:properties} and \ref{lemma:submod_3}, we can restrict the feasible region in problems (\ref{eqn:proposed_2}) and (\ref{eqn:proposed_3}) only across sets of cardinality $k$ where they are equivalent, and have the same optimal solution.
\end{proof}
\submodulargain*
\begin{proof}
Recall that for any set $\P$ with $|\P|\leq k$, $f(\P) \geq h(\P)$. As $|\hat{\P}|=k$, we have the equality $f\left(\hat{\P}\right) = h\left(\hat{\P}\right)$. Applying the classical greedy approximation theorem in \citep{nemhauser1978analysis}, we get $h\left(\hat{\P}\right) = f\left(\hat{\P}\right) \geq \left(1-1/e\right)f\left(\P^{\ast}\right) \geq \left(1-1/e\right) OPT$.
\end{proof}

\approximateMarginalGain*
\begin{proof}
At the iteration $i+1$, let $\bx^{\ast} = \argmax\limits_{\bx\in\S\setminus\P_{i}} \hat{f}(\bx|\P_{i})$ be the point that maximizes the approximate marginal gain function (\ref{eqn:approx-incremental-gain}), which is used to update the solution to $\P_{i+1} = \P_i \cup \{\bx^{\ast}\}$. Then,
\begin{equation}
\begin{aligned}
\label{eq:fincrementbound}
f\left(\bx^{\ast}|\P_i\right) \geq \hat{f}\left(\bx^{\ast}|\P_i\right) \geq\frac{1}{k} \sum\limits_{\bx_l \in \P^{\ast} \setminus \P_i} \left[\hat{f}\left(\bx_l|\P_i \right) \right]
\end{aligned}
\end{equation}
Further, for any $\bx_l \notin \P_{i}$ let $\P = \P_i \cup \{\bx_l\}$. We derive the inequality
\begin{equation*}
\begin{aligned}
f\left(\bx_l|\P_i\right) = f\left(\P\right) - f\left(\P_{i}\right) &= \inner{\bS\left(\I_{\P},:\right), \mgamma_{\P}\left(\I_{\P},:\right)} - \inner{\bS\left(\I_{\P_{i}},:\right), \mgamma_{\P_{i}}\left(\I_{\P_{i}},:\right)} \\
&\leq \inner{\bS\left(\I_{\P},:\right), \mgamma_{\P}\left(\I_{\P},:\right)} - \inner{\bS\left(\I_{\P_{i}},:\right), \mgamma_{\P}\left(\I_{\P_{i}},:\right)} \\
&=\inner{\bS(l,:), \mgamma_{\P}(l,:)}\\
&\leq \alpha_{l,\max}.
\end{aligned}
\end{equation*}
The inequality in the second line follows from the fact that $\mgamma_{\P_i}$ is the $\argmax$ for the set $\P_i$ in (\ref{eqn:proposed_3}) and $\mgamma_{\P}\left(\I_{\P_{i}},:\right)$ ---appended with $\bzero$ for other rows--- is one of the feasible solution. Likewise, $\hat{f}\left(\bx_l|\P_i\right) \geq \alpha_{l,\min}$. Hence,
\begin{equation}
\label{eq:fhatincrement}
    \hat{f}\left(\bx_l|\P_i\right) \geq \frac{\alpha_{l,\min}} {\alpha_{l,\max}} f\left(\bx_l|\P_i\right).
\end{equation}
Pugging the inequality (\ref{eq:fhatincrement}) in (\ref{eq:fincrementbound}) we have
\begin{equation}
\begin{aligned}
\label{eq:incgaininequality}
    f\left(\bx^{\ast}|\P_i\right) \geq \frac{1}{k} \sum\limits_{\bx_l \in \P^{\ast} \setminus \P_i} \left[\frac{\alpha_{l,\min}} {\alpha_{l,\max}} f\left(\bx_l|\P_i\right) \right] \geq \frac{\alpha}{k} \sum\limits_{\bx_l \in \P^{\ast} \setminus \P_i} \left[f\left(\bx_l|\P_i\right)\right].
\end{aligned}
\end{equation}
Leveraging the submodular and monotonic properties of the function $f(\cdot)$ from Lemma~\ref{lemma:submod_3}, we obtain the inequalities
\begin{equation}
\label{eq:submodpropusage}
    \sum\limits_{\bx_l \in \P^{\ast} \setminus \P_i} f\left(\bx_l|\P_i\right) \geq f\left(\P_i \cup \left(\P^*\setminus\P_i\right)\right) - f\left(\P_i\right) \geq f\left(\P^*\right)- f\left(\P_i\right).
\end{equation}
Noting that $f\left(\bx^{\ast}|\P_i\right) = \left[f\left(\P^\ast\right) - f\left(\P_i\right)\right] - \left[f\left(\P^\ast\right) - f\left(\P_{i+1}\right)\right]$, and using (\ref{eq:submodpropusage}) in (\ref{eq:incgaininequality}) gives the recurrence relation
\begin{equation*}
\begin{aligned}
f\left(\P^\ast\right) - f\left(\P_{i+1}\right) \leq \left(1- \frac{\alpha}{k}\right) \left[f\left(\P^\ast\right) - f\left(\P_{i}\right)\right],
\end{aligned}
\end{equation*}
frow which it follows that
\begin{equation*}
    f\left(\P^\ast\right) - f\left(\hat{\P}\right) \leq \left(1- \frac{\alpha}{k}\right)^k \left[f\left(\P^\ast\right) - f\left(\emptyset\right)\right].
\end{equation*}
As $f\left(\emptyset\right) = 0$, we get the desired result namely,
\begin{equation*}
   f\left(\hat{\P}\right) \geq \left(1-\left(1- \frac{\alpha}{k}\right)^k\right)f\left(\P^\ast\right) \geq \left(1-e^{-\alpha}\right)\mathrm{OPT}.
\end{equation*}
\end{proof}

\section{Implementation Details}\label{ref:implementationDetails}

\subsection{Hardware and License}
All models are implemented in \texttt{Python 3.10} using \texttt{PyTorch 2.3.0}. Image and language training are performed on servers with Intel(R) Xeon(R) Gold 6226R CPUs (2.90GHz) and three NVIDIA RTX A6000 GPUs. For language model pretraining, we use \texttt{JAX}~\citep{jax2018github} (v0.7.2) on the same GPU infrastructure.

\subsection{Algorithm Implementation}

\textbf{Implementation of POT Objective}:  To solve the entropic-regularized partial optimal transport (OT) problem, we rely on the Python Optimal Transport (\textsc{POT}) library\footnote{\url{https://pythonot.github.io/}.} . Specifically, we use the function \texttt{ot.partial.entropicwasserstein}\footnote{{\url{https://pythonot.github.io/gen_modules/ot.partial.html\#ot.partial.entropic_partial_wasserstein}. }},

which implements the entropic-regularized variant of partial OT. This formulation allows for transporting only a fraction of the total mass between the source and target distributions along with the enforcement of inequality on the marginals.

In our implementation, the cost matrix $C$ is constructed using pairwise distances between features of the source and target prototypes, which could be Euclidean or cosine distances depending on the application. The fraction of transported mass $\tau$ and the entropic regularization $\lambda$ are treated as hyperparameters. The function \texttt{ot.partial.entropic\_wasserstein} efficiently returns both the optimal transport plan and the associated partial OT cost, which we use as the objective function $f(\cdot)$ in downstream optimization or prototype selection procedures.

A typical usage in Python is as follows:
\begin{verbatim}
import ot

# mu: source weights
# nu: target weights
# C: cost matrix
# tau: fraction of transported mass
# lambda: entropy regularization
T, cost = ot.partial.entropic_wasserstein(
    mu, nu, C,
    tau=tau,
    reg=lambda
)
\end{verbatim}

The default maximum iterations parameter for this function is set adaptively along with a stopping threshold of $1e-6$.

\begin{table}[h!]
\centering
\caption{Source Size vs Maximum Iterations.}
\begin{tabular}{|c|c|}
\hline
\textbf{Source Set Size} & \textbf{Max Iterations} \\
\hline
64-200 & 100 \\
500-1000 & 1000 \\
1000-4000 & 2000 \\
5000 & 4000 \\
\hline
\end{tabular}
\end{table}

\subsection{Finetuning experiments}

We adapt the codebase of \citep{nguyen2024mini} for all our finetuning experiments. We use Adam \citep{adam2014method} with learning rate of 1e-5, gradient accumulation steps of 64 with batch size 1. We directly use raw LoRA gradients for constructing similarity matrices for CoLM, GREATS, $\methodprop$. For GREATS \citep{wang2024greats} we randomly sample few random points from train-set as anchors at every train step. For SBERT \citep{reimers2019sentence}, we use \textsc{Bert-base-uncased} as the embedding model, and construct similarity matrix from the embeddings instead of gradients.

\subsection{Details of baselines}
\paragraph{GREATS \citep{wang2024greats}.}
GREATS formulates online batch selection as optimizing a set utility that measures the single-step reduction in validation loss under a gradient-descent update. Let $w_t$ be the current parameters, $B_t$ a candidate batch, and $S \subseteq B_t$ a subset of size $k$. The ideal utility at iteration $t$ is
\[
\U^{(t)}(S; \bz^{(\mathrm{val})}) \;:=\; \ell(\btheta_t, \bz^{(\mathrm{val})}) \;-\; \ell\!\left(\btheta_t - \eta_t \sum_{\bz\in S} \nabla\ell(\btheta_t,\bz),\, \bz^{(\mathrm{val})}\right),
\]
and selection solves $\arg\max_{S \subseteq B_t,\,|S|=k} \U^{(t)}(S; \bz^{(\mathrm{val})})$.
Since exact evaluation is intractable, GREATS applies a lower-order Taylor approximation of the validation loss around $\btheta_t$ to obtain a closed-form surrogate for the marginal gain of adding a training point $\bz$:
\[
U^{(t)}(\bz \mid S)\;\approx\; \eta_t\, \bg(\bz)^\top \bg\!\big(\bz^{(\mathrm{val})}\big) \;-\; \eta_t^2\, \bg(\bz)^\top H\!\big(\bz^{(\mathrm{val})}\big)\, \bg(\bz^\ast),
\]
where $\bg(\cdot)=\nabla\ell(\btheta_t,\cdot)$, $\boldsymbol{\H}(\cdot)$ is the Hessian of the validation loss, and $z^\ast$ denotes the current aggregate. In practice, $\boldsymbol{\H}$ is approximated (e.g., $\boldsymbol{\H} \approx I$), yielding a gradient inner-product scoring with a correction term. A greedy procedure iteratively adds the point with largest approximate marginal gain until $k$ points are selected. To avoid materializing per-example model-sized gradients, GREATS computes all required gradient inner-products in a single backpropagation via a “ghost inner-product” reparameterization that expresses layerwise gradient inner-products using already-available activations and output gradients, and merges selection with the update without extra passes.

\paragraph{CoLM \citep{nguyen2024mini}.}
CoLM casts mini-batch construction as coreset selection in gradient space for memory-efficient fine-tuning. CoLM first addresses imbalance by including \emph{all} examples from “small” sources (those with insufficient sample count in the large batch), while selecting representatives (medoids) from each “big” source. To align selection with Adam, per-example gradients are normalized by the optimizer’s exponential-moving-average statistics, yielding normalized directions proportional to $\boldsymbol{m}_t/(\epsilon+\sqrt{\boldsymbol{v}_t})$. To reduce dimensionality and denoise, CoLM estimates the gradient of the \emph{last $V$-projection} parameters (e.g., LoRA $V$) using a zeroth-order SPSA estimator with two perturbed forward passes and precached penultimate activations, then sparsifies by keeping the coordinates with largest normalized magnitudes. Within each big source, a greedy medoid selection is performed in the projected, sparsified, Adam-normalized gradient space so that the aggregated coreset gradient approximates that of the full large batch; the final mini-batch is the union of all small-source examples and the selected big-source medoids.

\paragraph{SBERT \citep{reimers2019sentence}.}
SBERT modifies BERT into siamese/triplet architectures with shared weights that encode each sentence independently. A fixed-size embedding $u\in\mathbb{R}^d$ is obtained via a pooling operation over token representations (commonly mean pooling). Training uses sentence-pair supervision: (i) a classification objective on NLI pairs, where a classifier consumes a concatenation of functions of the two embeddings (e.g., $[u;\,v;\,|u-v|]$) to predict the label; (ii) a regression objective for semantic textual similarity, where the cosine of $(u,v)$ is regressed to a gold score via MSE; and (iii) optionally, triplet loss $\max\{0,\, \cos(u_a,u_n)-\cos(u_a,u_p)+\gamma\}$ for anchor–positive–negative tuples. At inference, sentence embeddings are compared with cosine or dot-product for retrieval and clustering.

\paragraph{GradNorm \citep{katharopoulos2018not}.}
Given a large batch $\mathcal{B}$, compute per-example gradient features and rank by norm. For parameters $\theta$ and loss $\ell_i=\ell(f_{\btheta}(\bx_i),y_i)$, define raw gradient $\boldsymbol{g}_i=\nabla_{\btheta} \ell_i$. To align with Adam, each coordinate is normalized as
\[\Tilde{\boldsymbol{g}_i} = \frac{\boldsymbol{m}_t}{\sqrt{\boldsymbol{v}_t}+\epsilon}\odot \boldsymbol{g}_i\]
where $(\boldsymbol{m}_t,\boldsymbol{v}_t)$ are the exponential moving averages of first and second moments. Instead of $\ell_2$ distance, GradNorm computes similarity in this normalized gradient space using cosine:
\[s_{ij} \;=\; \frac{\langle \Tilde{\boldsymbol{g}_i}, \Tilde{\boldsymbol{g}_j}\rangle}{\|\Tilde{ \boldsymbol{g}_i}\|_2 \,\|\Tilde{\boldsymbol{g}_j}\|_2}.\]
Each example is scored by its (smoothed) gradient norm $\|\Tilde{\boldsymbol{g}_i}\|_2$, and the top-$k$ are selected. This yields a subset whose update direction emphasizes examples with largest effective gradient magnitude under the optimizer’s scaling.

\paragraph{MaxLoss \citep{shalev2016minimizing}.}
Each example $i\in\mathcal{B}$ is scored by its instantaneous loss \(s_i = \ell(f_{\btheta}(x_i),y_i)\)
The $k$ highest-loss items are selected to form the training subset. This “hard-example” criterion requires only forward passes and captures points where the current model performs worst. Optionally, per-example losses can be combined with Adam-smoothed gradient norms to provide importance weights during optimization.

\subsection{Calculation of gradient features}
\label{sec:grad-feats}

Let $\bz_i$ denote an example, $\ell(\btheta; \bz_i)$ the training loss, and let $\bW^{(L)}_{V,\text{LoRA}}$ be the parameter tensor of the \emph{last} transformer block's value projection adapted by LoRA.\footnote{``Last'' refers to the topmost transformer block in the forward stack.} We flatten $\bW^{(L)}_{V,\text{LoRA}}$ to a vector $v \in \mathbb{R}^{d_{\mathrm{vp}}}$. At iteration $t$ and current parameters $\btheta_t$, we compute per-example gradients as
\[
\bg^{\mathrm{vp}}_{i,t} \;\;:=\;\; \nabla_{v}\,\ell(\btheta_t; \bz_i) \;\in\; \mathbb{R}^{d_{\mathrm{vp}}},
\]
where the gradients are restricted to the LoRA-adapted last $V$-projection. Unlike the zeroth-order MeZO estimator used in \citep{nguyen2024mini}, these gradients are obtained directly by backpropagation.

\paragraph{Adam-aligned normalization.}
To align with the update rule of Adam, we normalize each per-example gradient using the optimizer's moment statistics. Let $m_t, v_t \in \mathbb{R}^{d_{\mathrm{vp}}}$ denote the first and second moment accumulators,
\begin{equation}
\boldsymbol{m}_t \;=\; \beta_1 \boldsymbol{m}_{t-1} + (1-\beta_1)\,\bar \bg_t,\qquad
v_t \;=\; \beta_2 v_{t-1} + (1-\beta_2)\,\bar \bg_t^2,
\label{eq:adam-moments}
\end{equation}
with $\beta_1,\beta_2 \in (0,1)$, $\epsilon>0$, and $\bar \bg_t$ the average gradient over the current pool. The normalized gradient feature for an example $\bz_i$ is then
\begin{equation}
\boldsymbol{\phi}_{i,t} \;:=\; \frac{\boldsymbol{g}^{\mathrm{vp}}_{i,t}}{\epsilon+\sqrt{\boldsymbol{v}_t}} \;\in\; \mathbb{R}^{d_{\mathrm{vp}}},
\label{eq:adam-norm}
\end{equation}
where the division is elementwise. These features are stored and subsequently used in similarity computations.

\subsection{Pretraining Experiments}
\label{supp:pretraining}
We implement the pretraining setup using \texttt{JAX}~\citep{jax2018github}, primarily due to its just-in-time (JIT) compilation framework and empirical $2{-}3\times$ training speedups over \texttt{PyTorch}. For the base architecture, we pretrain a \textsc{LLaMA-3} model consisting of approximately $500$M parameters on the OpenWebText corpus\footnote{\url{https://huggingface.co/datasets/Skylion007/openwebtext}} \citep{radford2019language}. The training is carried out for $20$k training steps with an effective batch size of $64$ sequences, each of length $512$. Optimization is performed using the Adam algorithm with a fixed learning rate of $1 \times 10^{-4}$, without auxiliary learning rate schedules.  

For all methods involving prototype-based subset selection, we begin with a candidate batch of $128$ sequences and select $50\%$ prototypes, resulting in an effective batch size of $64$ sequences used for parameter updates. Subset selection is performed on a per-batch basis, without leveraging history across iterations.  

In the case of $\methodprop$, the underlying optimal transport (OT) problem is solved using $20$ Sinkhorn iterations with entropic regularization strength set to $1 \times 10^{-2}$. The similarity (cost) matrices are constructed directly from the last-layer gradients of the model, and no additional low-pass filtering, smoothing, or adaptive reweighting is applied. Throughout, cosine similarity is used as the base kernel to define pairwise affinities.  

The dataset is partitioned into a $95\%$ training split and a $5\%$ validation split.

\section{Experimental Setup Details}\label{ref:ExperimentalSetupDetails}

\subsection{Model Details}

\paragraph{\textsc{Phi-2} (2.7B).} 
\textsc{Phi-2} is a 2.7B parameter model trained with an emphasis on mathematical and logical reasoning, derived from curated synthetic corpora and filtered web data. It supports a context length of 2{,}048 tokens. In our fine-tuning setup, we apply LoRA adapters (rank 128, $\alpha=512$, dropout 0.05) to all attention projection matrices (QKV) and the two feed-forward layers.

\paragraph{\textsc{Phi-3} family.} 
We experiment primarily with the 3.8B variant (\textsc{Phi-3 Mini}), though the broader family also includes 7B and 14B models. The \textsc{Phi-3} series continues the focus on compact models optimized for reasoning tasks, with available context lengths of 4K and 128K tokens depending on variant. Similar to \textsc{Phi-2}, we apply LoRA adapters to QKV projections and feed-forward layers during fine-tuning.

\paragraph{\textsc{StableLM Zephyr 3B}.} 
\textsc{StableLM Zephyr 3B} is a 3B parameter instruction-tuned model designed as a general-purpose assistant, without a specific emphasis on mathematical reasoning. It supports input sequences up to 4K tokens. For LoRA fine-tuning, we insert adapters into all attention projection matrices (QKVO).

\subsection{Datasets}
For image settings we do on MNIST, CIFAR10, CIFAR100
as well as synthetic distributions. 

For the mathematical reasoning experiments, we fine-tune on the \textbf{\textsc{MathInstruct}} dataset \citep{yue2023mammoth}, which contains roughly 260K instruction–response pairs. The data is aggregated from 14 open-source mathematics corpora, covering diverse subfields and spanning a broad range of difficulty levels. The composition of MathInstruct is highly imbalanced—the largest constituent source is nearly 300 times larger than the smallest—and the detailed distribution across sources is provided in Figure~4a of the Appendix. Prior work has shown that fine-tuning on MathInstruct leads to state-of-the-art results on multiple standardized mathematical reasoning benchmarks.

For classification experiments, we additionally use three datasets from the \textsc{SuperGLUE} benchmark \citep{wang2019superglue}: SST-2, CB, and MultiRC. For CB, we retain the complete training set of  250 labeled examples. For SST-2 and MultiRC, we randomly subsample 3K  examples each for training. 

\subsection{Training Details}

Following the configuration in \citet{yue2023mammoth}, we employ a learning rate of $2 \times 10^{-5}$ with a cosine decay scheduler. The learning rate is linearly warmed up from 0 to $2 \times 10^{-5}$ during the first 3\% of training steps and subsequently decays to 0 following a cosine schedule. We fix the maximum sequence length to 512 tokens. Unless otherwise stated, all experiments on \textsc{MathInstruct} are trained for the equivalent of 1K gradient update steps. To enable larger effective batch sizes, we use gradient accumulation with an accumulation factor of 8.

For parameter-efficient fine-tuning, we adopt LoRA \citep{hu2022lora} with rank 128, scaling parameter $\alpha=512$, and a dropout rate of 0.05. On \textsc{Phi} models, LoRA adapters are applied to all attention projection matrices (QKV) as well as the two feed-forward layers. On \textsc{Zephyr}, we apply LoRA to all attention projections (QKVO). All experiments are conducted on 3 NVIDIA A6000 GPUs, and each configuration is repeated three times to account for variance in training.

\subsection{Evaluation Datasets and Metrics}

Following \citet{yue2023mammoth}, we evaluate our models on a diverse suite of mathematical reasoning benchmarks spanning both in-domain and out-of-domain distributions. 

\paragraph{In-domain benchmarks.} The in-domain evaluation covers three widely used datasets: \textsc{GSM8K} \citep{}, \textsc{MATH} \citep{hendrycks2021math}, and \textsc{NumGLUE} \citep{mishra2022numglue}. \textsc{GSM8K} focuses on grade-school arithmetic word problems, \textsc{MATH} contains high-school competition-style problems across 29 mathematical domains, and \textsc{NumGLUE} extends natural language understanding tasks with quantitative reasoning components.

\paragraph{Out-of-domain benchmarks.} To test generalization beyond the training distribution, we additionally include \textsc{SVAMP}, the \textsc{Mathematics} dataset, and \textsc{SimulEq}. These datasets emphasize robustness across algebraic manipulations, probability and statistics, number theory, and systems of equations, while also incorporating instances requiring multi-step logical reasoning and commonsense knowledge.

\subsection{Evaluation Setup}

\textbf{1NN Classifier Setup}\label{app:image-exp_details}
Let $\X$ and $\Y$ denote the source and target datasets, respectively, with potentially differing class distributions, and let $\P \subseteq \X$ be a candidate representative set for the target dataset $\Y$. The quality of $\P$ is assessed using a 1-nearest neighbour (1-NN) classifier parameterized by the elements of $\P$. Each instance $y \in Y$ is assigned the label of its nearest prototype in $\P$, where the ground-truth class labels of the elements in $\P$ are assumed to be available during this evaluation. The resulting classification accuracy serves as the evaluation metric for comparing prototype selection algorithms.

\textbf{LLM-Finetuning:} All questions are posed in an open-ended format. We adopt the standard \textit{exact match} metric, where a prediction is considered correct only if it exactly matches the gold reference solution. Evaluation is conducted under the $0$-shot setting with a maximum decoding context length of 2048 tokens. We use the Program-of-Thought (PoT) prompting strategy as the default, and fall back to Chain-of-Thought (CoT) prompting when PoT is not applicable, following \citet{yue2023mammoth}.

\section{Additional Experimental Results}
\label{supp:additional}

\subsection{Additional Ablations on Entropic Regularization}\label{supp:ablation_entropy}

To examine the influence of the entropic regularization parameter on the approximation error associated with the computation of the approximate marginal gain in Eq.~(\ref{eqn:approx-incremental-gain}), we conduct an ablation study under a controlled synthetic setting derived from the CIFAR dataset. Specifically, we uniformly sample 2000 training instances from CIFAR and treat them identically as both the source and target sets, denoted by $\S$ and $\T$, respectively. As illustrated in Fig.~\ref{fig:entropic_reg_ablation}, increasing the value of $\lambda$ progressively reduces the approximation error toward zero, with the ratio between the approximate and exact marginal gains converging to unity and exhibiting low variance.

\begin{figure}[h!]
    \centering
    \begin{subfigure}[t]{0.23\textwidth}
        \centering
        \includegraphics[width=\linewidth]{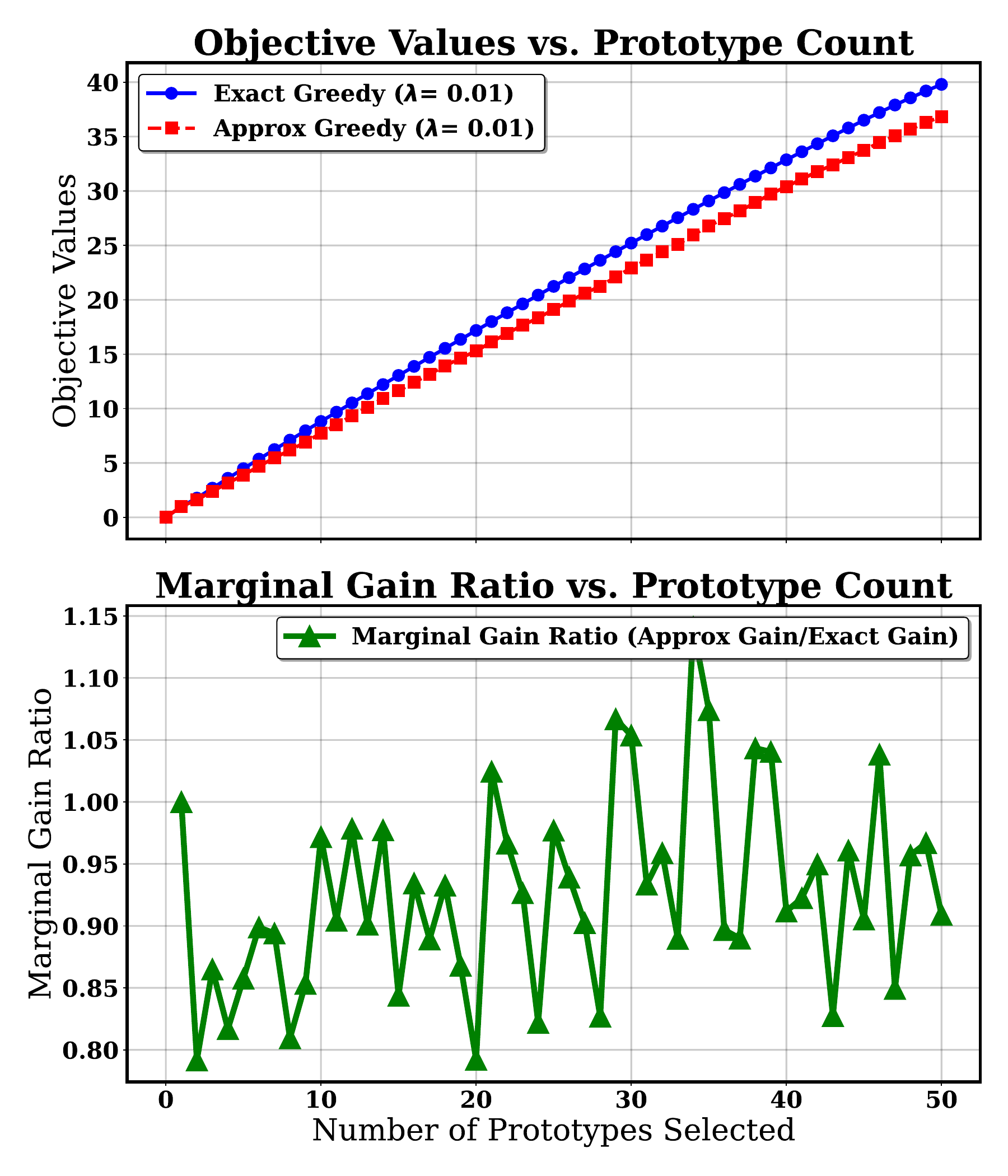}
        \caption{Entropic Reg. ($\lambda = 0.01$)}
    \end{subfigure}
    \hfill
    \begin{subfigure}[t]{0.23\textwidth}
        \centering
        \includegraphics[width=\linewidth]{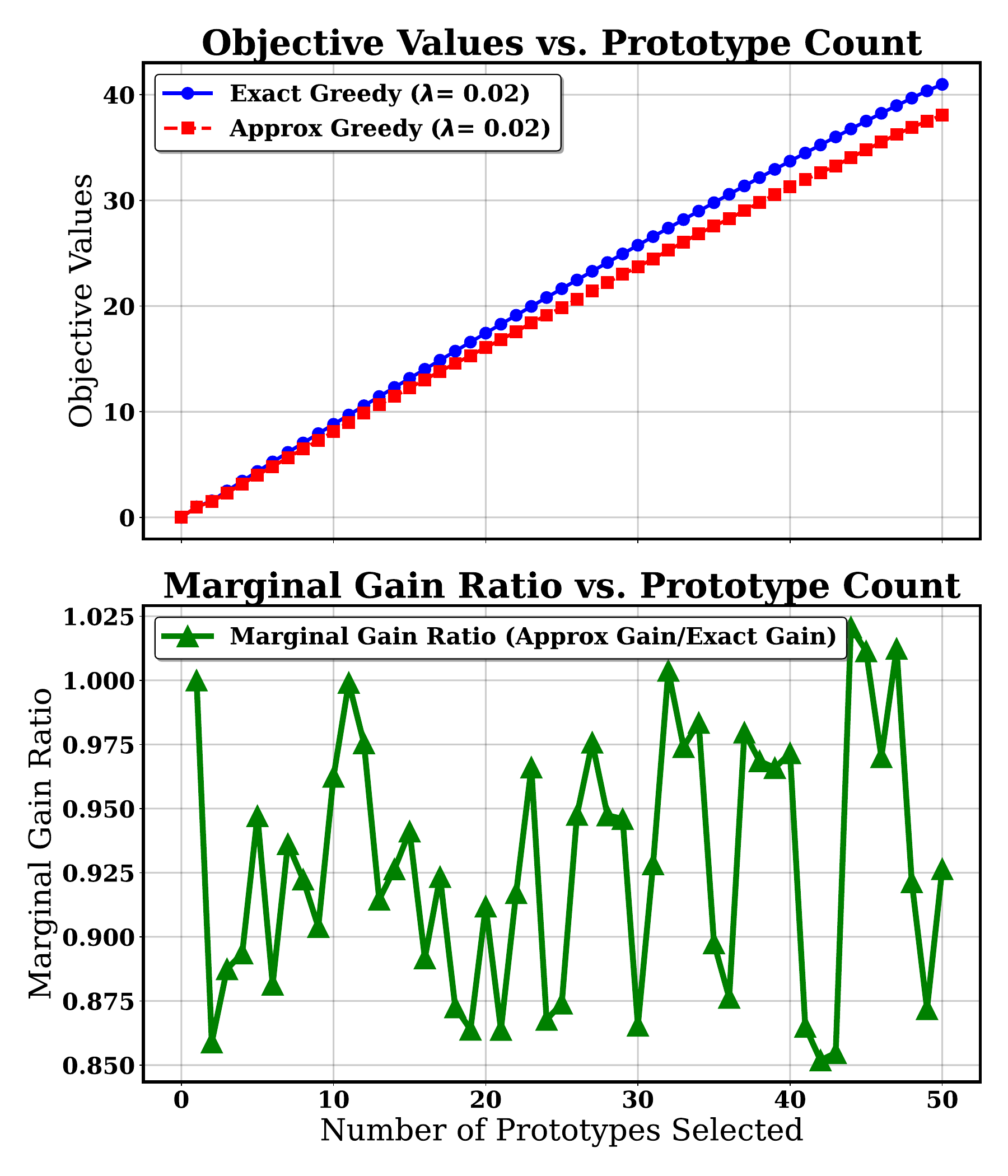}
        \caption{Entropic Reg. ($\lambda = 0.02$)}
    \end{subfigure}
    \hfill
    \begin{subfigure}[t]{0.23\textwidth}
        \centering
        \includegraphics[width=\linewidth]{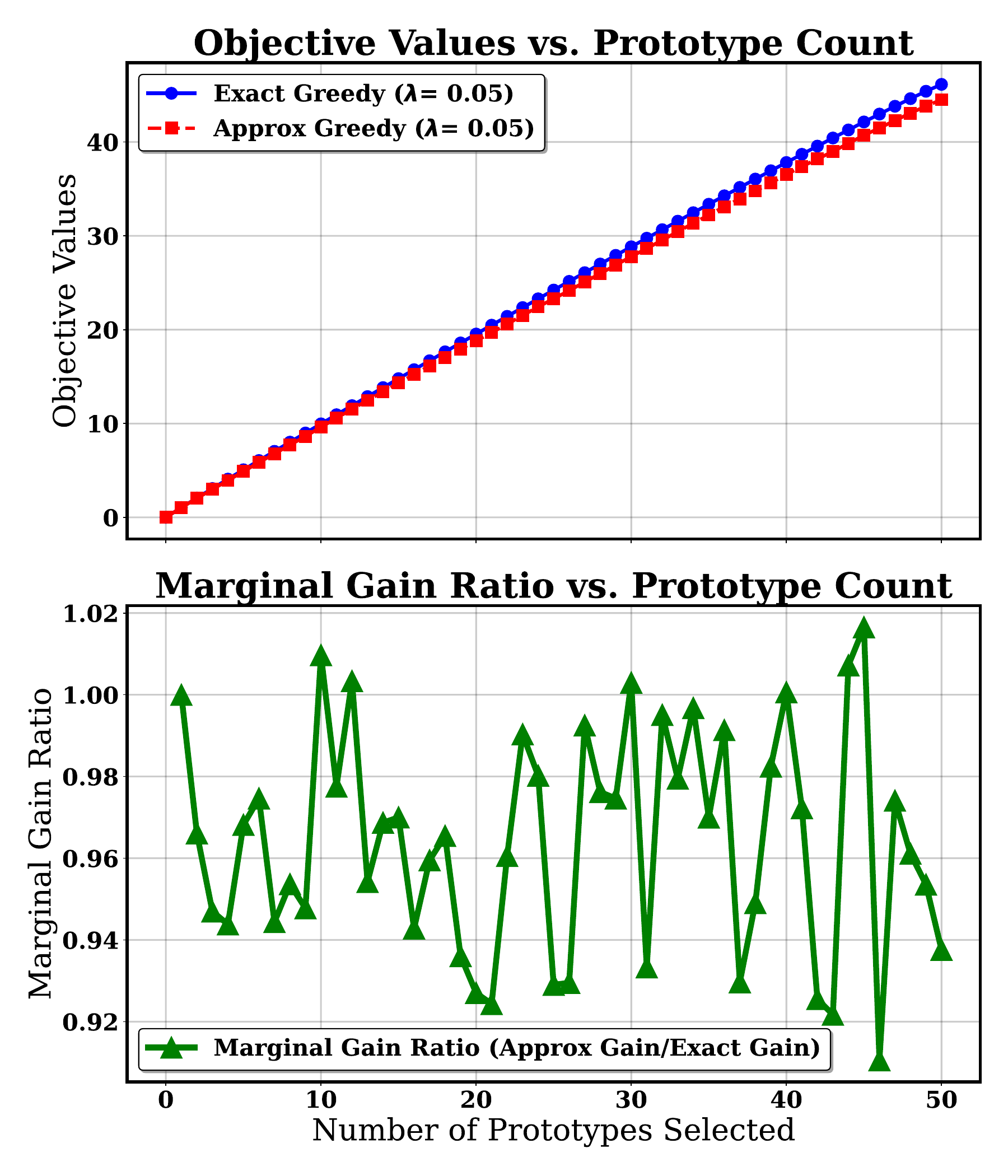}
        \caption{Entropic Reg. ($\lambda = 0.05$)}
    \end{subfigure}
    \hfill
    \begin{subfigure}[t]{0.23\textwidth}
        \centering
        \includegraphics[width=\linewidth]{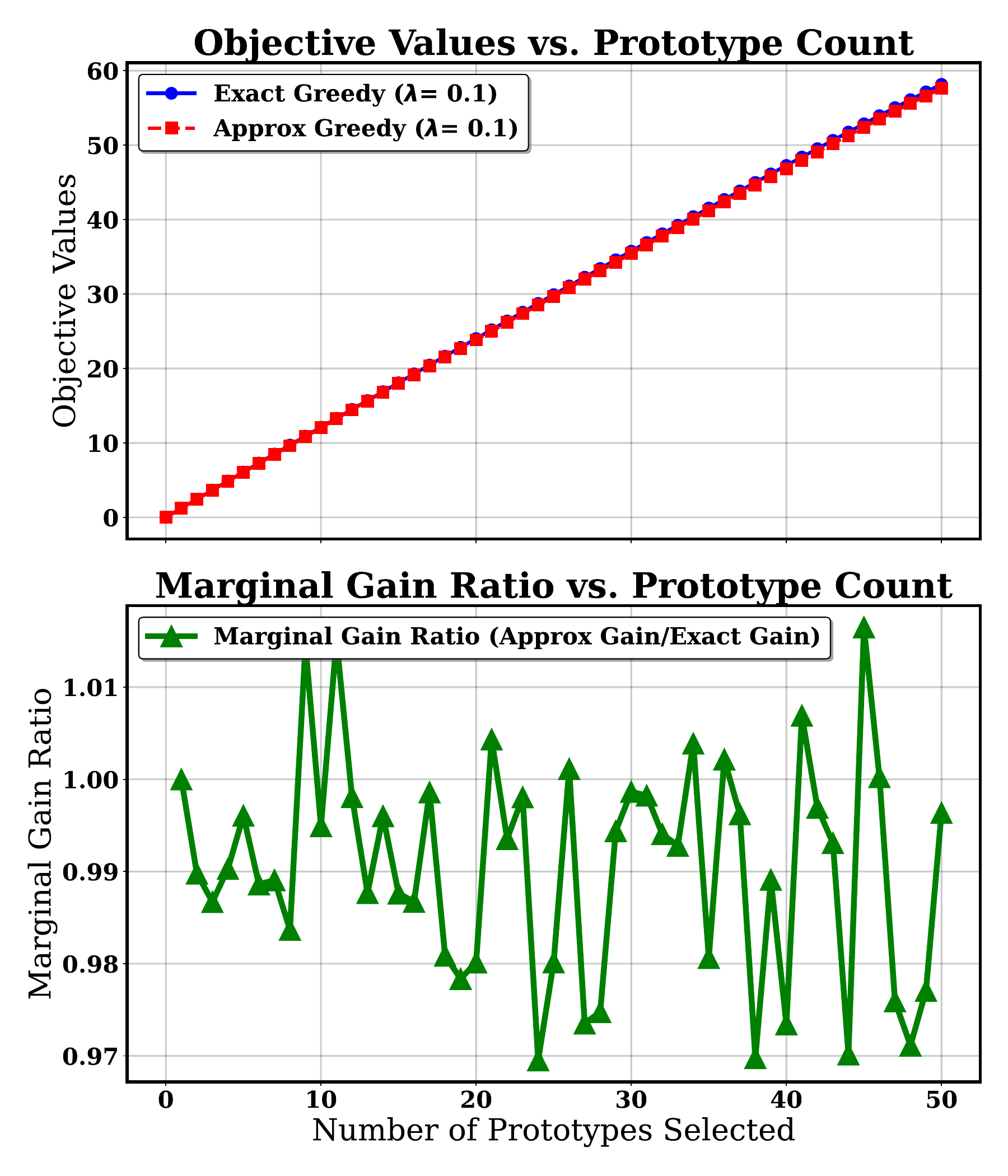}
        \caption{Entropic Reg. ($\lambda = 0.1$)}
    \end{subfigure}

    \caption{Approximate vs. actual marginal gain as a function of the entropic regularization parameter $\lambda$. Increasing $\lambda$ reduces the approximation error, with the marginal gain ratio converging to 1 and exhibiting low variance for larger $\lambda$ values.}
    \label{fig:entropic_reg_ablation}
\end{figure}

\subsection{Additional Results on $\methodprop$-PB}
\subsubsection{Experiments on batch size=256 for  $\methodprop$-PB}
\begin{figure}[!h]
  \centering
  \begin{minipage}[t]{0.58\linewidth}
    \vspace{0pt}
    \small
    We experiment with $\methodprop$-PB batch size of 256 for selection instead of source-wise prototype selection. We fine-tune \textsc{Phi-3} for 2048 steps (selection batch size 256) with prototype percentage 0.25, resulting in an effective batch size of 32. Table \ref{table:supp_full} shows that $\methodprop$ continues to be effective even in the full-batch setting.

    We compare against GREATS, GradNorm, and CoLM under the same training budget and report validation log-perplexity trajectories over optimization steps. Figure \ref{fig:Supp_bs256} shows that CoLM's validation perplexity degrades as batch size increases, while $\methodprop$-PB remains stable and continues to improve perplexity throughout training.

    We report additional results on $\methodprop$-PB on \textbf{\textsc{MathInstruct}} dataset.
  \end{minipage}\hfill
  \begin{minipage}[t]{0.40\linewidth}
    \vspace{0pt}
    \centering    \input{iclr2026/plots/plot_bs256}
\refstepcounter{figure}\label{fig:Supp_bs256}
    \vspace{0.25em}
    \parbox{\linewidth}{\raggedleft\footnotesize Figure~\thefigure: Validation perplexity when batch size=256 and prototype ratio is 25\%.\par}
  \end{minipage}
  \vspace{-0.5em}
\end{figure}

\subsubsection{Experiments on full-batch prototype selection}
For this variant, we experiment full-batch selection instead of source wise prototype selection. We finetune \textsc{Phi-3} for 2048 steps, selection batch size of 128, with prototype percentage as 0.5, resulting in a effective batch of 64. Table \ref{table:supp_full} shows that \methodprop continues to be effective even in full batch setting.

\begin{table}[h!]
    \centering
    \small
    \setlength{\tabcolsep}{6pt}
    \renewcommand{\arraystretch}{1.2}
    \caption{Comparison of $\methodprop$-PB performance for $|\B|=256$, following \citet{yue2023mammoth} on \textbf{\textsc{MathInstruct}} dataset.}
    \resizebox{0.92\textwidth}{!}{
    \begin{tabular}{l S S S S S S S}
        \toprule
        \textbf{Method} & \textbf{Avg} & \multicolumn{3}{c}{\textbf{In-domain}} & \multicolumn{3}{c}{\textbf{Out-of-domain}} \\
        \cmidrule(lr){3-5} \cmidrule(lr){6-8}
        & & \textbf{GSM8K} & \textbf{MATH} & \textbf{NumGLUE} & \textbf{SVAMP} & \textbf{Mathematics} & \textbf{SimulEq} \\
        \midrule
        COLM \citep{nguyen2024mini}     
            & 59.86 
            & 74.13 & 36.70 & 63.14 
            & 86.50 & 36.40 & 62.30 \\
        GREATS \citep{wang2024greats}   
            & 60.79 
            & 78.62 & 37.90 & 63.90 
            & 85.50 & 36.90 & 61.90 \\
        \textbf{$\methodprop$-PB (Ours)} 
            & \textbf{61.34} 
            & 78.20 & 37.60 & 66.03 
            & 84.90 & 37.70 & 63.60 \\
        \bottomrule
    \end{tabular}}
    \label{table:supp_full}
\end{table}

\begin{figure*}[htp!]
    \centering

    \small
    \setlength{\tabcolsep}{6pt}
    \renewcommand{\arraystretch}{1.2}
    \captionof{table}{Per-batch performance across in-domain and out-of-domain datasets for \textsc{Phi-3} on \textbf{\textsc{MathInstruct}}, batch size $|\B|=128$ and total budget $k=64$. We compare FT, GradNorm, COLM, GREATS, and $\methodprop$ (Ours).}
    \label{table:main_batchonly}
    \resizebox{0.90\textwidth}{!}{%
    \begin{tabular}{l c c c | c | c c c | c | c}
        \toprule
        \textbf{Method} & \multicolumn{4}{c}{\textbf{In-domain}} & \multicolumn{4}{c}{\textbf{Out-of-domain}} & \textbf{Avg-All} \\
        \cmidrule(lr){2-5} \cmidrule(lr){6-9}
        & \textbf{GSM8K} & \textbf{MATH} & \textbf{NumGLUE} & \textbf{Avg} 
        & \textbf{SVAMP} & \textbf{Mathematics} & \textbf{SimulEq} & \textbf{Avg} 
        &  \\
        \midrule
        FT (bs=64) & 76.72 & 36.54 & 62.57 & 58.61
                   & 85.10 & 33.30 & 62.78 & 60.39
                   & 59.50 \\
        GradNorm & 75.40 & 35.03 & 64.10 & 58.18
                  & 84.17 & 36.50 & 65.70 & 62.12
                  & 60.15 \\
        COLM & 76.36 & 36.42 & 64.10 & 58.96
             & 85.30 & 37.40 & 63.60 & 62.10
             & 60.53 \\
        GREATS & 77.80 & 37.28 & 64.40 & 59.83
               & 85.00 & 38.00 & 62.06 & 61.69
               & 60.76 \\
        $\methodprop$ (Ours) & 78.16 & 37.76 & 66.02 & \textbf{60.65}
             & 85.70 & 37.20 & 68.28 & \textbf{63.73}
             & \textbf{62.19} \\
        \bottomrule
    \end{tabular}}

    \vspace{1.2em}

    \begin{minipage}[t]{\textwidth}
        \centering
        \scalebox{0.6}{\input{iclr2026/plots/plot_bs128_full}}
    \end{minipage}

    \caption{Top: Per-batch performance of FT, GradNorm, COLM, GREATS, and $\methodprop$ on \textsc{MathInstruct} with $|\B|=128$, $k=64$.  
    Bottom: Validation perplexity when $|\B|=128$, subset ratio 50\%, and prototype selection is batch-wise.}
    \label{fig:main_combo_batchonly}
\end{figure*}

\subsection{Additional Results on Zephyr-3B}
Here, we report additional results on \textsc{Zephyr-3B} on \textsc{MathInstruct} Dataset.
    
\begin{table*}[h!]
    \centering
    \small
    \setlength{\tabcolsep}{6pt}
    \renewcommand{\arraystretch}{1.2}
    \caption{Performance across in-domain and out-of-domain datasets for \textsc{Zephyr-3B} on \textbf{\textsc{MathInstruct}}, batch size $|\B|=128$ and total budget $k=32$. Results are reported under two configurations for \textbf{all} baselines: \emph{source-wise} (left of ``/'') and \emph{batch-wise} (right of ``/'').}
    \label{table:main}
    \resizebox{0.98\textwidth}{!}{%
    \begin{tabular}{l c c c | c | c c c | c | c}
        \toprule
        \textbf{Method} & \multicolumn{4}{c}{\textbf{In-domain}} & \multicolumn{4}{c}{\textbf{Out-of-domain}} & \textbf{Avg-All} \\
        \cmidrule(lr){2-5} \cmidrule(lr){6-9}
        & \textbf{GSM8K} & \textbf{MATH} & \textbf{NumGLUE} & \textbf{Avg}
        & \textbf{SVAMP} & \textbf{Mathematics} & \textbf{SimulEq} & \textbf{Avg}
        &  \\
        \midrule
        FT & 52.08 & 16.14 & 40.10 & 36.11
           & 54.70 & 15.60 & 22.20 & 30.83
           & 33.47 \\
        GradNorm & 52.3 \,/\, 49.8 & 16.5 \,/\, 15.7 & 40.0 \,/\, 39.9 & 35.70
           & 53.5 \,/\, 53.1 & 15.5 \,/\, 15.4 & 22.5 \,/\, 22.7 & 30.45
           & 33.07 \\
        SBERT & 51.3 \,/\, 21.7 & 14.6 \,/\, 14.01 & 41.6 \,/\, 27.44 & 28.44
           & 56.3 \,/\, 34.5 & 16.2 \,/\, 13.6 & 21.98 \,/\, 13.8 & 26.06
           & 27.25 \\
        COLM & 50.6 \,/\, 49.88 & \textbf{21.42} \,/\, 17.4 & 40.15 \,/\, 39.9 & 36.56
           & \textbf{55.8} \,/\, \textbf{54.1} & 16.7 \,/\, 15.4 & 22.17 \,/\, 21.01 & 30.86
           & 33.71 \\
        GREATS & 52.8 \,/\, 50.6 & 19.01 \,/\, 15.9 & 40.8 \,/\, 38.8 & 36.32
           & 54.5 \,/\, \textbf{54.6} & \textbf{17.1} \,/\, 13.9 & 21.98 \,/\, 21.78 & 30.64
           & 33.48 \\
        $\methodprop$ (Ours) & \textbf{54.4} \,/\, \textbf{54.89} & 20.4 \,/\, \textbf{19.6} & \textbf{41.3} \,/\, \textbf{40.4} & \textbf{38.50}
           & 54.3 \,/\, 54.1 & 16.5 \,/\, \textbf{15.5} & \textbf{24.7} \,/\, \textbf{22.95} & \textbf{31.34}
           & \textbf{34.92} \\
        \bottomrule
    \end{tabular}}
\end{table*}

\begin{figure}[htbp]
 \centering
 \begin{minipage}[t]{\textwidth}
 \centering \scalebox{0.6}{\input{iclr2026/plots/plot_bs128_32_zeph}}
 \end{minipage}
 \caption{Zephyr-3b Validation perplexity when bs=128 and subset ratio 25\% and prototype selection is batch wise.}
\end{figure}
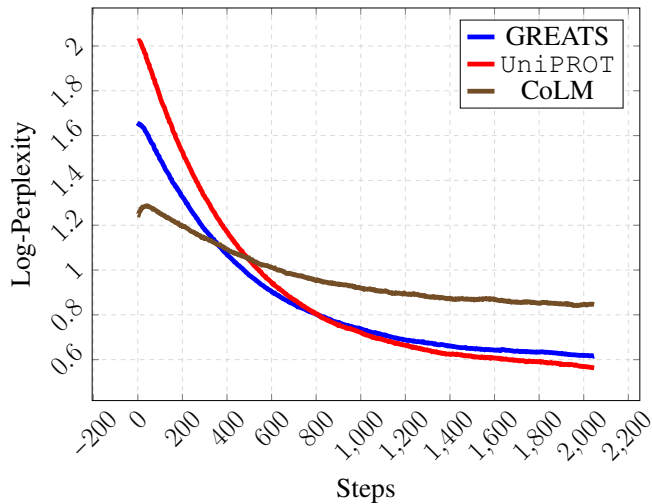


\subsection{Additional Results on \textsc{Phi-2}}

Here, we report additional results on \textsc{Phi-2} on \textsc{MathInstruct} Dataset.

\begin{table*}[h!]
    \centering
    \small
    \setlength{\tabcolsep}{6pt}
    \renewcommand{\arraystretch}{1.2}
    \caption{Performance across in-domain and out-of-domain datasets for \textsc{Phi-2} on \textbf{\textsc{MathInstruct}}, batch size $|\B|=128$ and total budget $k=32$. Results are reported under two configurations for \textbf{all} baselines: \emph{source-wise} (left of ``/'') and \emph{batch-wise} (right of ``/'').}
    \label{table:main}
    \resizebox{0.98\textwidth}{!}{%
    \begin{tabular}{l c c c | c | c c c | c | c}
        \toprule
        \textbf{Method} & \multicolumn{4}{c}{\textbf{In-domain}} & \multicolumn{4}{c}{\textbf{Out-of-domain}} & \textbf{Avg-All} \\
        \cmidrule(lr){2-5} \cmidrule(lr){6-9}
        & \textbf{GSM8K} & \textbf{MATH} & \textbf{NumGLUE} & \textbf{Avg}
        & \textbf{SVAMP} & \textbf{Mathematics} & \textbf{SimulEq} & \textbf{Avg}
        &  \\
        \midrule
        FT & 59.28 & 14.80 & 47.88 & 40.65
           & \textbf{62.80} & 19.40 & 36.70 & 39.63
           & 40.14 \\
        GradNorm & 58.6 \,/\, 60.3 & 15.1 \,/\, 14.6 & 51.3 \,/\, 50.6 & 41.75
           & 61 \,/\, 59 & 20.2 \,/\, \textbf{20.4} & 40.6 \,/\, 38.13 & 39.89
           & 40.82 \\
        SBERT & 58.4 \,/\, 45.5 & 12.43 \,/\, 9.83 & 52.01 \,/\, 48.4 & 37.76
           & 61 \,/\, \textbf{64.2} & 18.1 \,/\, 16.6 & 36.5 \,/\, 33.8 & 38.37
           & 38.06 \\
        COLM & 57.01 \,/\, 58.8 & 14.86 \,/\, 14.6 & 51.05 \,/\, 50.8 & 41.19
           & 60 \,/\, 63.9 & 18.9 \,/\, 19.8 & 32.4 \,/\, 31.9 & 37.82
           & 39.50 \\
        GREATS & 59.5 \,/\, 58.4 & 15.4 \,/\, 15.6 & 51.2 \,/\, \textbf{54.1} & 42.37
           & \textbf{61.3} \,/\, 61.5 & \textbf{21.1} \,/\, 20.7 & 37.35 \,/\, 35.01 & 39.49
           & 40.93 \\
        $\methodprop$ (Ours) & \textbf{60.6} \,/\, 58.8 & \textbf{16.8} \,/\, \textbf{16.7} & \textbf{51.61} \,/\, 51.7 & \textbf{42.70}
           & 61 \,/\, 62 & 20.9 \,/\, 19.8 & \textbf{41.4} \,/\, \textbf{35.6} & \textbf{40.12}
           & \textbf{41.41} \\
        \bottomrule
    \end{tabular}}
\end{table*}

\begin{table}[h!]
    \centering
    \small
    \setlength{\tabcolsep}{6pt}
    \renewcommand{\arraystretch}{1.2}
    \caption{Comparison of \textsc{Phi-2} \textbf{per-source} selection performance across in-domain and out-of-domain datasets, following \citet{yue2023mammoth} on \textbf{\textsc{MathInstruct}} dataset, batch size $|\B|=256$ and total budget $k=32$.}
    \resizebox{0.92\textwidth}{!}{
    \begin{tabular}{l S S S S S S S}
        \toprule
        \textbf{Method} & \textbf{Avg} & \multicolumn{3}{c}{\textbf{In-domain}} & \multicolumn{3}{c}{\textbf{Out-of-domain}} \\
        \cmidrule(lr){3-5} \cmidrule(lr){6-8}
        & & \textbf{GSM8K} & \textbf{MATH} & \textbf{NumGLUE} & \textbf{SVAMP} & \textbf{Mathematics} & \textbf{SimulEq} \\
        \midrule
        COLM-PS \citep{nguyen2024mini}     
            & 43.14 
            & 61.03 & 26.67 & 52.65 
            & 60.45 & 21.04 & 37.00 \\
        GREATS-PS \citep{wang2024greats}   
            & 43.73 
            & 61.92 & 27.21 & 52.40 
            & 62.00 & 20.80 & 38.03 \\
        \textbf{$\methodprop$-PS (Ours)} 
            & \textbf{44.66} 
            & 62.40 & 27.63 & 53.97 
            & 64.72 & 20.30 & 38.91 \\
        \bottomrule
    \end{tabular}}
\end{table}

\section{Additional Related Works}

\paragraph{OT for representation matching without selection.}
\citet{wang2025pot} employ optimal transport (OT) to align class activation map (CAM) clusters with corresponding class prototypes. However, their use of OT is limited to \emph{distribution matching} between pixel-level feature distributions and pre-defined cluster prototypes, rather than for subset or prototype selection. In particular, the prototypes are constructed as averages of pixel features assigned to each cluster (cf.~Eq.~(2) in \citealp{wang2025pot}), and are not selected from a candidate pool under any combinatorial or cardinality constraint. Thus, no subset selection or prototype selection problem is formulated or solved in their framework.

\citet{zhang2024hyperspherical} also leverage OT to compute token-to-prototype assignments by matching representations to their corresponding ground-truth prototypes. Unlike subset selection approaches, their formulation retains the full set of prototypes during optimization and uses entropy-regularized OT to obtain \emph{soft assignments}. As a result, their method does not impose any sparsity, selection, or cardinality constraints on the set of active prototypes. Consequently, OT serves purely as an assignment mechanism rather than a tool for selecting a representative subset.

\paragraph{OT in coreset selection and dataset condensation.}
OT has also been explored in coreset selection and dataset condensation settings, where the goal is to construct a compact dataset that approximates the full data distribution. For example, Wasserstein-based coreset construction and dataset distillation methods~\citep{zhao2021dataset,nguyen2021dataset,liu2025dataset} leverage OT distances to measure distributional fidelity between synthetic and real data. However, these approaches typically \emph{learn synthetic samples} or optimize continuous representations, rather than selecting a subset from a given discrete pool. As a result, they differ fundamentally from subset selection problems with combinatorial constraints.

\paragraph{OT barycenters and prototype learning.}
Another related direction involves OT barycenters, where prototypes are obtained as Wasserstein means of distributions~\citep{cuturi2014fast,agueh2011barycenters}. These methods compute representative prototypes in a continuous space by averaging distributions under OT geometry. While effective for summarization, they do not select prototypes from an existing dataset and thus bypass the combinatorial nature of subset selection. Similarly, OT-based clustering methods rely on iterative refinement of centroids rather than discrete selection from a candidate set \citep{ho2017multilevel}.

In contrast to the above works, our formulation uses \emph{partial optimal transport} POT \emph{as a selection principle}, where the objective explicitly optimizes over a constrained subset of prototypes under uniform weighting and cardinality restrictions. This leads to a fundamentally different regime in which OT directly governs subset selection, rather than serving solely as a matching or assignment operator.

\paragraph{Distinction to ~\citep{hong2024diversified}}
~\citep{hong2024diversified} study mini-batch selection for efficient training of deep networks by selecting a subset of points from a given batch that maximizes group-wise orthogonalized representativeness. This objective is fundamentally different from our optimal transport (OT) formulation, which seeks to minimize the discrepancy between a prototypical distribution and a target distribution. In particular, their method operates \emph{within} a single batch and does not incorporate any notion of a separate target set. Consequently, it is not directly applicable to settings where one aims to select a subset $\P \subseteq \S$ (from a source set $\S$) such that it matches a target set $\T$.

Algorithmically, their DivBS procedure performs greedy selection by iteratively choosing samples that maximize projection onto the current residual. This results in a weighted expansion of the full gradient, where each selected sample is assigned an importance coefficient based on its marginal contribution. Crucially, these coefficients are \emph{non-uniform}: early selections capture dominant components (e.g., the mean), while subsequent selections explain diminishing residuals. Although the final mini-batch is treated uniformly during backpropagation (via averaging), the selection mechanism itself is inherently non-uniform and does not enforce equal importance among selected samples.

\paragraph{Our approach (UniPROT).}
In contrast, our method explicitly incorporates a \emph{uniform weight constraint} into the optimization objective. We select prototypes by solving a constrained OT problem that minimizes the transport cost to the target distribution under the requirement that each prototype carries equal mass. This leads to a fundamentally different selection principle that ensures equal-importance representation while aligning the selected subset with the target distribution.

\section{Broader Impact}\label{ref:BroaderImpact}
This work develops a principled framework for prototype selection that aims to improve fairness and robustness in settings with distributional imbalance. By explicitly enforcing uniform weighting, \methodprop can reduce systematic under-representation of minority classes, which has positive implications for equitable model performance across demographic or domain groups. At the same time, more efficient subset selection methods could also be leveraged to accelerate training of harmful or biased systems if applied without safeguards. We believe that open discussion of both the benefits and limitations of prototype selection methods is important to ensure they are deployed responsibly, and that continued transparency in this line of work will help maximize positive societal impact.

\section{Code}\label{ref:Code}

We release our code on \href{https://github.com/efficiency-learning/UniPROT
}{GitHub}.

%% file: iclr2026/plots/plot_bs256.tex
\begin{tikzpicture}
    \begin{axis}[
        width=0.92\linewidth,
        height=0.70\linewidth,
        xlabel={Steps},
        ylabel={Log-Perplexity},
        grid=both,
        grid style={dashed, gray!30},
        font=\normalsize,
        label style={font=\normalsize},
        tick label style={font=\normalsize},
        title style={font=\normalsize},
        legend columns=2,
        legend style={
            font=\scriptsize,
            at={(0.5,1.03)},
            anchor=south,
            cells={align=left},
            /tikz/every even column/.append style={column sep=1em},
            draw=none,
            fill=white,
            fill opacity=0.85,
            text opacity=1
        }
    ]
        \addplot+[mark=none, line width=\plotwidth] table [x index=0, y index=1] {iclr2026/data/greats_bs256.dat};
        \addlegendentry{GREATS}
        \addplot+[mark=none, line width=\plotwidth] table [x index=0, y index=1] {iclr2026/data/uniprot_bs256.dat};
        \addlegendentry{$\methodprop$-PB}
        \addplot+[mark=none, line width=\plotwidth] table [x index=0, y index=1] {iclr2026/data/colm_bs256.dat};
        \addlegendentry{CoLM}
        \addplot+[mark=none, line width=\plotwidth] table [x index=0, y index=1] {iclr2026/data/gradnorm_bs256.dat};
        \addlegendentry{GradNorm}
    \end{axis}
\end{tikzpicture}

%% file: iclr2026/plots/plot_bs128_full.tex
\begin{tikzpicture}
    \begin{axis}[
        width=0.8\textwidth,
        height=0.6\textwidth,
        xlabel={Steps},
        ylabel={Log-Perplexity},
        grid=both,
        grid style={dashed, gray!30},
        legend pos=north east,
        font=\LARGE, 
        label style={font=\LARGE}, 
        tick label style={font=\LARGE,rotate=45}, 
        title style={font=\LARGE}, 
        legend style={font=\LARGE} 
    ]
        \addplot+[mark=none, line width=\plotwidth] table [x index=0, y index=1] {iclr2026/data/full_greats.dat};
        \addlegendentry{GREATS}
        \addplot+[mark=none, line width=\plotwidth] table [x index=0, y index=1] {iclr2026/data/full_uniprot.dat};
        \addlegendentry{$\methodprop$-PB}
        \addplot+[mark=none, line width=\plotwidth] table [x index=0, y index=1] {iclr2026/data/full_colm.dat};
        \addlegendentry{CoLM}
        \addplot+[mark=none, line width=\plotwidth] table [x index=0, y index=1] {iclr2026/data/full_gradnorm.dat};
        \addlegendentry{GradNorm}
        \addplot+[mark=none, line width=\plotwidth] table [x index=0, y index=1] {iclr2026/data/full_random.dat};
        \addlegendentry{FT}

    \end{axis}
\end{tikzpicture}

%% file: iclr2026/plots/plot_bs128_32_zeph.tex
\begin{tikzpicture}
    \begin{axis}[
        width=0.8\textwidth,
        height=0.6\textwidth,
        xlabel={Steps},
        ylabel={Log-Perplexity},
        grid=both,
        grid style={dashed, gray!30},
        legend pos=north east,
        font=\LARGE, 
        label style={font=\LARGE}, 
        tick label style={font=\LARGE, rotate=45}, 
        title style={font=\LARGE}, 
        legend style={font=\LARGE} 
    ]
        \addplot+[mark=none, line width=\plotwidth] table [x index=0, y index=1] {iclr2026/data/bs128_32_greats_zeph.dat};
        \addlegendentry{GREATS}
        \addplot+[mark=none, line width=\plotwidth] table [x index=0, y index=1] {iclr2026/data/bs128_32_uniprot_zeph.dat};
        \addlegendentry{$\methodprop$}
        \addplot+[mark=none, line width=\plotwidth] table [x index=0, y index=1] {iclr2026/data/bs128_32_colm_zeph.dat};
        \addlegendentry{CoLM}
        
    \end{axis}
\end{tikzpicture}